\documentclass{article}

\usepackage{caption}
\usepackage{subcaption}

\usepackage[nonatbib,final]{nips_2018}
\usepackage{xpatch}
\usepackage[utf8]{inputenc} 
\usepackage[T1]{fontenc}    
\usepackage{hyperref}       
\usepackage{url}            
\usepackage{booktabs}       
\usepackage{amsfonts}       
\usepackage{nicefrac}       
\usepackage{microtype}      
\usepackage{lipsum}
\usepackage{caption}
\usepackage[utf8]{inputenc} 
\usepackage[T1]{fontenc}    
\usepackage{hyperref}       
\usepackage{url}            
\usepackage{booktabs}       
\usepackage{amsfonts}       
\usepackage{nicefrac}       
\usepackage{microtype}      
\usepackage{times}
\usepackage{tabularx,tabulary}
\usepackage{graphicx}
\usepackage{amsthm}
\usepackage{amsmath}
\usepackage{amssymb}
\usepackage{enumitem}
\usepackage{algorithm}
\usepackage{algorithmic}  

\newtheorem{ex}{Example}

\newtheorem{thm}{Theorem}
\newtheorem*{thm*}{Theorem}

\newtheorem{lemma}{Lemma}
\newtheorem{cor}{Corollary}

\usepackage[compact]{titlesec}

\newcommand*\dif{\mathop{}\!\mathrm{d}}

\DeclareMathOperator{\conv}{conv}
\DeclareMathOperator{\JS}{\mathrm{JSD}}
\DeclareMathOperator{\KL}{\mathrm{KL}}
\DeclareMathOperator{\OT}{\mathrm{OT}}

\makeatletter
\xpatchcmd{\proof}{\topsep6\p@\@plus6\p@\relax}{}{}{}
\makeatother

\title{A Convex Duality Framework for GANs}

\author{
  Farzan Farnia$^*$ \\
  \texttt{farnia@stanford.edu}  \And David Tse\thanks{Department of Electrical Engineering, Stanford University.}\\
  \texttt{dntse@stanford.edu} \\
}

\begin{document}

\maketitle

\begin{abstract}
Generative adversarial network (GAN) is a minimax game between a generator mimicking the true  model and a discriminator distinguishing the samples produced by the generator from the real training samples. Given an unconstrained discriminator able to approximate any function, this game reduces to finding the generative model minimizing a divergence measure, e.g. the Jensen-Shannon (JS) divergence, to the data distribution. However, in practice the discriminator is constrained to be in a smaller class $\mathcal{F}$ such as neural nets. Then, a natural question is how the divergence minimization interpretation changes as we constrain $\mathcal{F}$. In this work, we address this question by developing a convex duality framework for analyzing GANs. For a convex set $\mathcal{F}$, this duality framework interprets the original GAN formulation as finding the generative model with minimum JS-divergence to the distributions penalized to match the moments of the data distribution, with the moments specified by the discriminators in $\mathcal{F}$. We show that this interpretation more generally holds for f-GAN and Wasserstein GAN. As a byproduct, we apply the duality framework to a hybrid of f-divergence and Wasserstein distance. Unlike the f-divergence, we prove that the proposed hybrid divergence changes continuously with the generative model, which suggests regularizing the discriminator's Lipschitz constant in f-GAN and vanilla GAN. We numerically evaluate the power of the suggested regularization schemes for improving GAN's training performance.

\end{abstract}

\section{Introduction}
Learning a probability model from data samples is a fundamental task in unsupervised learning.  
The recently developed generative adversarial network (GAN) \cite{goodfellow2014generative} leverages the power of deep neural networks to successfully address this task across various domains \cite{goodfellow2016nips}. 
In contrast to traditional methods of parameter fitting like maximum likelihood estimation, the GAN approach views the problem as a {\em game}  between a {\it generator} $G$ whose goal is to generate fake samples that are close to the real data training samples and a {\it discriminator} $D$ whose goal is to distinguish between the real and fake samples. The generator creates the fake samples by mapping from random noise input.

The following minimax problem is the original GAN problem, also called \emph{vanilla GAN}, introduced in \cite{goodfellow2014generative}
\begin{equation} \label{GAN: Goodfellow}
\min_{G \in \mathcal{G}}\: \max_{D \in \mathcal{F}}\: \mathbb{E}\bigl[\log D({\mathbf{X}})\bigr] + \mathbb{E}\bigl[\log\bigl( 1 - D(G(\mathbf{Z}))\bigr)\bigr].
\end{equation}
Here $\mathbf{Z}$ denotes the generator's noise input, $\mathbf{X}$ represents the random vector for the real data distributed as $P_\mathbf{X}$, and $\mathcal{G}$ and $\mathcal{F}$ respectively represent the generator and discriminator function sets. Implementing this minimax game using deep neural network classes $\mathcal{G}$ and $\mathcal{F}$  has lead to the state-of-the-art generative model for many different tasks.

To shed light on the probabilistic meaning of vanilla GAN, \cite{goodfellow2014generative} shows that given an unconstrained discriminator $D$, i.e. if $\mathcal{F}$ contains all possible functions, the minimax problem \eqref{GAN: Goodfellow} will reduce to 
 \begin{equation} \label{GAN: Goodfellow_distributonal}
\min_{G \in \mathcal{G}}\, \JS(P_{\mathbf{X}}, P_{G(\mathbf{Z})}),
\end{equation}
where $ \JS$ denotes the Jensen-Shannon (JS) divergence. The optimization problem \eqref{GAN: Goodfellow_distributonal} can be interpreted as finding the closest generative model to the data distribution $P_{\mathbf{X}}$ (Figure \ref{Figure 1}a), where distance is measured using the JS-divergence. Various GAN formulations were later proposed by changing the divergence measure in \eqref{GAN: Goodfellow_distributonal}: f-GAN \cite{nowozin2016f} generalizes vanilla GAN by minimizing a general f-divergence
; Wasserstein GAN (WGAN) \cite{arjovsky2017wasserstein} considers the first-order Wasserstein (the earth-mover's) distance;  MMD-GAN \cite{dziugaite2015training,li2015generative,li2017mmd} considers the maximum mean discrepancy; Energy-based GAN \cite{zhao2016energy} minimizes the total variation distance as discussed in \cite{arjovsky2017wasserstein}; Quadratic GAN \cite{feizi2017understanding} finds the distribution minimizing the second-order Wasserstein distance.

\begin{figure}[t]
\vskip 0.2in
\begin{center}
\centerline{\includegraphics[width=0.7\columnwidth]{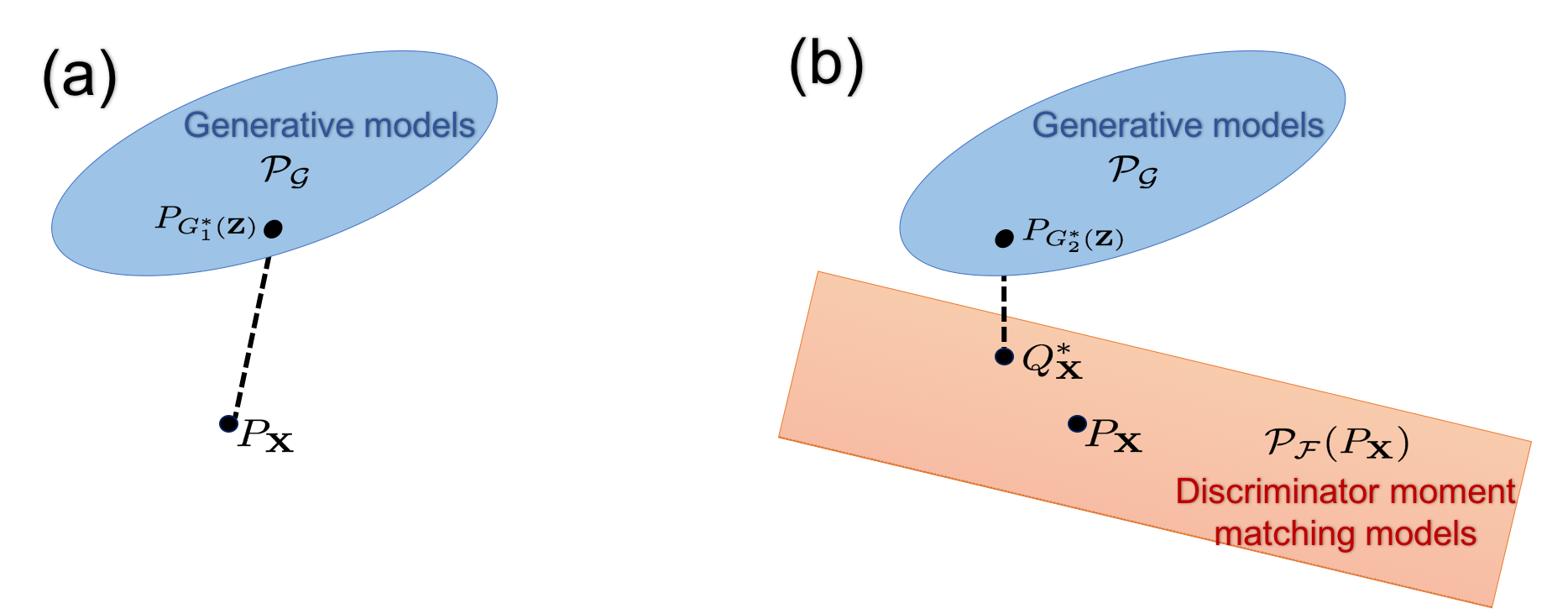}}
\caption{(a) Divergence minimization in \eqref{GAN: Goodfellow_distributonal} between $P_{\mathbf{X}}$ and generative models $\mathcal{P_G}$ for unconstrained $\mathcal{F}$, (b) Divergence minimization in \eqref{GAN: Goodfellow_distributonal_constarined} between generative models $\mathcal{P_G}$ and discriminator moment matching models $\mathcal{P_F}(P_{\mathbf{X}})$.}
\label{Figure 1}
\end{center}
\vskip -0.2in
\end{figure}

However, GANs trained in practice differ from this minimum divergence formulation, since their discriminator is not optimized over an unconstrained set and is constrained to smaller classes such as neural nets. As shown in \cite{arora2017generalization}, constraining the discriminator is in fact necessary to guarantee good generalization properties for GAN's learned model. Then, how does the minimum divergence interpretation \eqref{GAN: Goodfellow_distributonal} change as we constrain $\mathcal{F}$? A standard approach used in \cite{arora2017generalization,liu2017approximation} is to view the maximum discriminator objective as an $\mathcal{F}$-based distance between distributions. For unconstrained $\mathcal{F}$, the $\mathcal{F}$-based distance reduces to the original divergence measure, e.g. the JS-divergence in vanilla GAN. 

While $\mathcal{F}$-based distances have been shown to be useful for analyzing GAN's generalization properties \cite{arora2017generalization}, their connection to the original divergence measure remains unclear for a constrained $\mathcal{F}$. Then, what is the interpretation of GAN minimax game with a constrained discriminator? In this work, we address this question by interpreting the dual problem to the discriminator optimization. To analyze the dual problem, we develop a convex duality framework for general divergence minimization problems. We apply the duality framework to the f-divergence and optimal transport cost families, providing interpretation for f-GAN, including vanilla GAN minimizing JS-divergence, and Wasserstein GAN. 

Specifically, we generalize the interpretation for unconstrained $\mathcal{F}$ in \eqref{GAN: Goodfellow_distributonal} to any linear space discriminator set $\mathcal{F}$. For this class of discriminator sets, we interpret vanilla GAN as the following JS-divergence minimization between two sets of probability distributions, the set of generative models and the set of discriminator moment-matching distributions (Figure \ref{Figure 1}b),
 \begin{equation} \label{GAN: Goodfellow_distributonal_constarined}
\min_{G \in \mathcal{G}}\; \min_{Q \in \mathcal{P_F}(P_\mathbf{X})}\, \JS (P_{G(\mathbf{Z})},Q).
\end{equation} 
Here $ \mathcal{P_F}(P_\mathbf{X})$ contains any distribution $Q$ satisfying the moment matching constraint $\mathbb{E}_Q[D(\mathbf{X})]=\mathbb{E}_P[D(\mathbf{X})]$ for all discriminator $D$'s in $\mathcal{F}$. 
More generally, we show that a similar interpretation applies to GANs trained over any convex discriminator set $\mathcal{F}$. We further discuss the application of our duality framework to neural net discriminators with bounded Lipschitz constant. While a set of neural network functions is not necessarily convex, we prove any convex combination of Lipschitz-bounded neural nets can be approximated by uniformly combining boundedly-many neural nets. This result applied to our duality framework suggests considering a uniform mixture of multiple neural nets as the discriminator. 

As a byproduct, we apply the duality framework to the minimum sum hybrid of f-divergence and the first-order Wasserstein ($W_1$) distance, e.g. the following hybrid of JS-divergence and $W_1$ distance:
\begin{equation}\label{hybrid: JS, W1}
d_{\JS , W_1} (P_1,P_2) := \min_{Q}\; W_1(P_1,Q) + \JS(Q,P_2). 
\end{equation}
We prove that this hybrid divergence enjoys a continuous behavior in distribution $P_1$. Therefore, the hybrid divergence provides a remedy for the discontinuous behavior of the JS-divergence when optimizing the generator parameters in vanilla GAN. \cite{arjovsky2017wasserstein} observes this issue with the JS-divergence in vanilla GAN and proposes to instead minimize the continuously-changing $W_1$ distance in WGAN. However, as empirically demonstrated in \cite{miyato2018spectral} vanilla GAN with Lipschitz-bounded discriminator remains the state-of-the-art method for training deep generative models in several benchmark tasks.
Here, we leverage our duality framework to prove that the hybrid $d_{\JS , W_1} $, which possesses the same continuity property as in $W_1$ distance, is in fact the divergence measure minimized in vanilla GAN with $1$-Lipschitz discriminator. Our analysis hence provides an explanation for why regularizing the discriminator's Lipschitz constant via gradient penalty \cite{gulrajani2017improved} or spectral normalization \cite{miyato2018spectral} improves the training performance in vanilla GAN. We then extend our focus to the hybrid of f-divergence and the second-order Wasserstein ($W_2$) distance. In this case, we derive the f-GAN (e.g. vanilla GAN) problem with its discriminator being adversarially trained using Wasserstein risk minimization \cite{sinha2018certifiable}. We numerically evaluate the power of these families of hybrid divergences in training vanilla GAN.

\section{Divergence Measures}\label{Section: Divergence}
\subsection{Jensen-Shannon divergence}
The Jensen-Shannon divergence is defined in terms of the KL-divergence (denoted by $\KL$) as
\begin{equation*}
\JS(P, Q) := \frac{1}{2} \KL(P {\Vert}  M) + \frac{1}{2} \KL(Q {\Vert} M)
\end{equation*}
where $M=\frac{P+Q}{2}$ is the mid-distribution between $P$ and $Q$. Unlike the KL-divergence, the JS-divergence is symmetric $\JS(P, Q) =\JS(Q, P) $ and bounded $0\le \JS(P, Q)\le 1$. 
\subsection{f-divergence}
The f-divergence family \cite{csiszar2004information} generalizes the KL and JS divergence measures. Given a convex lower semicontinuous function $f$ with $f(1)=0$, the f-divergence $d_f$ is defined as
\begin{equation}\label{f-div def}
d_f(P,Q) := \mathbb{E}_{P}\bigl[ f\bigl(\frac{q(\mathbf{X})}{p(\mathbf{X})}\bigr) \bigr] = \int p(\mathbf{x}) 
f \bigl(\frac{q(\mathbf{x})}{p(\mathbf{x})}\bigr) \dif \mathbf{x}.
\end{equation}
Here $\mathbb{E}_{P}$ denotes expectation over distribution $P$ and $p,\, q$ denote the density functions for distributions $P,\, Q$, respectively. 
The KL-divergence and the JS-divergence are members of the f-divergence family, corresponding to respectively $f_{\KL}(t)=t\log t$ and $f_{\JS}(t)=\frac{t}{2}\log t- \frac{t+1}{2}\log \frac{t+1}{2}$. 
\subsection{Optimal transport cost, Wasserstein distance}
The optimal transport cost for cost function $c(\mathbf{x},\mathbf{x}')$, which we denote by ${\OT}_c$, 
is defined as 
\begin{equation}\label{Wasserstein-c defnition}
{\OT}_c(P,Q) := \inf_{
M \in \Pi(P,Q) }
\mathbb{E} \bigl[ c(\mathbf{X},\mathbf{X}') \bigr],
\end{equation}
where $\Pi(P,Q)$ contains all couplings with marginals $P,Q$. The Kantorovich duality \cite{villani2008optimal} shows that for a non-negative lower semi-continuous cost $c$,
\begin{equation}\label{Kantorovich duality}
{\OT}_c(P,Q) = \max_{D\, \text{c-concave}} \mathbb{E}_P\bigl[D(\mathbf{X})\bigr] -  \mathbb{E}_Q\bigl[ D^c(\mathbf{X})\bigr],
\end{equation}
where we use $D^c$ to denote $D$'s c-transform defined as
$D^c(\mathbf{x}) := \sup_{\mathbf{x}'}\:  D(\mathbf{x}') - c(\mathbf{x},\mathbf{x}')$ and call $D$ c-concave if $D$ is the c-transform of a valid function. Considering the norm-based cost $c_q(\mathbf{x},\mathbf{x}')= \Vert \mathbf{x} - \mathbf{x}' \Vert^q$ with $q\ge 1$, the $q$th order Wasserstein distance $W_q$ is defined based on the $c_q$ optimal transport cost as
\begin{equation}\label{def: Wasserstein distance}
W_q(P,Q) := OT_{c_q}(P,Q)^{1/q} =  \inf_{
M \in \Pi(P,Q) }
\mathbb{E} \bigl[\, \Vert\mathbf{X} - \mathbf{X}'\Vert^q \,\bigr]^{1/q}.
\end{equation}
An important special case is the first-order Wasserstein ($W_1$) distance  corresponding to the difference norm cost $c_1(\mathbf{x},\mathbf{x}') = \Vert\mathbf{x}-\mathbf{x}'\Vert$. Given cost function $c_1$, a function $D$ is c-concave if and only if $D$ is $1$-Lipschitz, and the c-transform $D^{c} = D$ for any $1$-Lipschitz $D$. Therefore, the Kantorovich duality \eqref{Kantorovich duality} implies that
\begin{equation}
W_1(P,Q) = \max_{D\, \text{1-Lipschitz}} \mathbb{E}_P\bigl[D(\mathbf{X})\bigr] -  \mathbb{E}_Q\bigl[ D(\mathbf{X})\bigr].
\end{equation}
Another notable special case is the second-order Wasserstein ($W_2$) distance, corresponding to the difference norm-squared cost $c_2(\mathbf{x},\mathbf{x}') =\Vert \mathbf{x} -\mathbf{x}'\Vert^2$.


\section{Divergence minimization in GANs: a convex duality framework}
In this section, we develop a convex duality framework for analyzing  divergence minimization problems conditioned to moment-matching constraints. Our framework generalizes the duality framework developed in \cite{altun2006unifying} for the f-divergence family.

For a general divergence measure $d(P,Q)$, we define $d$'s conjugate over distribution $P$, which we denote by $d^*_{P}$, as the following mapping from real-valued functions of $\mathbf{X}$ to real numbers
\begin{equation}\label{Conjugate: def}
d^*_{P}(D)\, := \, \sup_{Q }\: \mathbb{E}_{Q}[D(\mathbf{X})] - d(P,Q) .
\end{equation}
Here the supremum is over all distributions on $\mathbf{X}$ with support set $\mathcal{X}$. We later show the following theorem, which is based on the above definition, recovers various well-known GAN formulations, when applied to divergence measures discussed in Section \ref{Section: Divergence}. 
\begin{thm}\label{Thm: dual max disc_general}
Suppose divergence $d(P,Q)$ is non-negative, lower semicontinuous and convex in distribution $Q$. Consider a convex set of continuous functions $\mathcal{F}$ and assume support set $\mathcal{X}$ is compact. Then, 
\begin{align}\label{Eq: min-min GAN unregularized_Lagrangian}
 &\min_{G\in\mathcal{G}}\;\: \max_{D \in \mathcal{F}}\;\:   \mathbb{E}_{P_{\mathbf{X}}}[D(\mathbf{X})] - d^*_{P_{G(\mathbf{Z})}}(D) \\  
 = \,& \min_{G\in\mathcal{G}}\; \min_{Q}\; \bigl\{ d( P_{G(\mathbf{Z})} , Q ) + \max_{D\in\mathcal{F}}\,\{ \,\mathbb{E}_{P_\mathbf{X}}[D(\mathbf{X})] - \mathbb{E}_{Q}[D(\mathbf{X})]\,\}  \bigr\}. \nonumber
\end{align}
\end{thm}
\begin{proof}
We defer the proof to the Appendix.
\end{proof}
Theorem \ref{Thm: dual max disc_general} interprets \eqref{Eq: min-min GAN unregularized_Lagrangian}'s LHS minimax problem as searching for the closest generative model to the distributions penalized to share the same moments specified by $\mathcal{F}$ with $P_\mathbf{X}$. The following corollary of Theorem \ref{Thm: dual max disc_general} shows if we further assume that $\mathcal{F}$ is a linear space, then 
the penalty term penalizing moment mismatches can be moved to the constraints. This reduction reveals a divergence minimization problem between generative models and the following set $\mathcal{P_F}(P)$ which we call the set of discriminator moment matching distributions,
\begin{equation}\label{P_F difinition}
\mathcal{P_F}(P) := \bigl\{  Q:\; \forall D\in\mathcal{F},\: \mathbb{E}_Q[D(\mathbf{X})]= \mathbb{E}_P[D(\mathbf{X})]\, \bigr\}. 
\end{equation}
\begin{cor}\label{Cor: linear}
In Theorem \ref{Thm: dual max disc_general} suppose $\mathcal{F}$ is further a linear space, i.e. for any $D_1,D_2\in \mathcal{F}$ and $\lambda\in\mathbb{R}$ we have $D_1+\lambda D_2\in \mathcal{F}$. Then,
\begin{align}\label{Eq: dual min-max GAN unregularized}
\min_{G\in\mathcal{G}}\; \max_{D \in \mathcal{F}} \;  \mathbb{E}_{P_{\mathbf{X}}}[D(\mathbf{X})] -d^*_{P_{G(\mathbf{Z})}}(D)  \, = \, \min_{G\in\mathcal{G}}\; \min_{Q\in \mathcal{P}_\mathcal{F}(P_\mathbf{X})} \; d( P_{G(\mathbf{Z})} , Q)  .
\end{align}
\end{cor}
In next section, we apply this duality framework to divergence measures discussed in Section \ref{Section: Divergence} and show how to derive various GAN problems through the developed framework.
\section{Duality framework applied to different divergence measures}
\subsection{f-divergence: f-GAN and vanilla GAN}
Theorem \ref{Theorem f_gan} shows the application of Theorem \ref{Thm: dual max disc_general} to f-divergences. We use $f^*$ to denote $f$'s convex-conjugate \cite{boyd2004convex}, defined as $f^*(u) := \sup_t ut-f(t)$. Note that Theorem \ref{Theorem f_gan} applies to any f-divergence $d_f$ with non-decreasing convex-conjugate $f^*$, which holds for all f-divergence examples discussed in \cite{nowozin2016f} with the only exception of Pearson $\chi^2$-divergence.
\begin{thm}\label{Theorem f_gan}
Consider f-divergence $d_f$ where the corresponding $f$ has a non-decreasing convex-conjugate $f^*$. In addition to Theorem \ref{Thm: dual max disc_general}'s assumptions, suppose $\mathcal{F}$ is closed to adding constant functions, i.e. $D+\lambda\in\mathcal{F}$ if $D\in\mathcal{F},\, \lambda\in\mathbb{R}$. Then, the minimax problem in the LHS of \eqref{Eq: min-min GAN unregularized_Lagrangian} and \eqref{Eq: dual min-max GAN unregularized}, will reduce to 
\begin{align}\label{f-GAN dual}
 \min_{ G\in \mathcal{G} }\;  \max_{D \in \mathcal{F}}\;  \mathbb{E}[D({\mathbf{X}})] - \mathbb{E}\bigl[f^* \bigl( D(G(\mathbf{Z})) \bigr)\bigr]. 
\end{align}
\end{thm}
\begin{proof}
We defer the proof to the Appendix.
\end{proof}
The minimax problem \eqref{f-GAN dual} is in fact the f-GAN problem \cite{nowozin2016f}. Theorem \ref{Theorem f_gan} hence reveals that f-GAN searches for the generative model minimizing f-divergence to the distributions matching moments specified by $\mathcal{F}$ to the moments of true distribution. 
\begin{ex}
Consider the JS-divergence, i.e. f-divergence corresponding to $f_{\JS}(t)=\frac{t}{2}\log t - \frac{t+1}{2}\log\frac{t+1}{2}$. Then, \eqref{f-GAN dual} up to additive and multiplicative constants reduces to
\begin{equation}\label{JS-GAN dual}
\min_{ G\in \mathcal{G} }\;  \max_{D \in \mathcal{F}}\;  \mathbb{E}[D({\mathbf{X}})] + \mathbb{E}\bigl[ \log\bigl( 1 - \exp( D(G(\mathbf{Z}))\bigr)\bigr].
\end{equation}
Moreover, if for function set $\tilde{\mathcal{F}}$ the corresponding $\mathcal{F}=\{D:\, D(\mathbf{x})= - \log(1+ \exp(\tilde{D}(\mathbf{x}))),\, \tilde{D}\in \tilde{\mathcal{F}}\}$ is a convex set, then \eqref{JS-GAN dual} will reduce to the following minimax game which is the vanilla GAN problem \eqref{GAN: Goodfellow} with sigmoid activation applied to the discriminator output,
\begin{equation}\label{JS-GAN dual, GAN}
\min_{ G\in \mathcal{G} }\:  \max_{\tilde{D} \in \tilde{\mathcal{F}}}\:  \mathbb{E}\bigl[\, \log \frac{1}{1+\exp(\tilde{D}(\mathbf{X}))}\, \bigr] + \mathbb{E}\bigl[\, \log \frac{\exp(\tilde{D}(\mathbf{X}))}{1+\exp(\tilde{D}(\mathbf{X}))}\, \bigr] .
\end{equation}
\end{ex}
\subsection{Optimal Transport Cost: Wasserstein GAN}
\begin{thm}\label{thm: Wasseaarstein}
Let divergence $d$ be optimal transport cost ${\OT}_c$ where $c$ is a non-negative lower semicontinuous cost function. Then, the minimax problem in the LHS of \eqref{Eq: min-min GAN unregularized_Lagrangian} and \eqref{Eq: dual min-max GAN unregularized} reduces to 
\begin{align}\label{Wc-GAN dual}
 \min_{ G\in \mathcal{G} }\;  \max_{D \in \mathcal{F}}\;  \mathbb{E}[D({\mathbf{X}})] - \mathbb{E}\bigl[ D^c(G(\mathbf{Z})) \bigr]. 
\end{align}
\end{thm}
\begin{proof}
We defer the proof to the Appendix.
\end{proof}
Therefore the minimax game between $G$ and $D$ in \eqref{Wc-GAN dual} can be viewed as minimizing the optimal transport cost between generative models and the distributions matching moments over $\mathcal{F}$ with $P_\mathbf{X}$'s moments. The following example applies this result to the first-order Wasserstein distance and recovers the WGAN problem \cite{arjovsky2017wasserstein} with a constrained $1$-Lipschitz discriminator.  
\begin{ex}
Let the optimal transport cost in \eqref{Wc-GAN dual} be the $W_1$ distance, and suppose $\mathcal{F}$ is a convex subset of 1-Lipschitz functions. Then, the minimax problem \eqref{Wc-GAN dual} will reduce to
\begin{equation}\label{W1-GAN dual}
\min_{ G\in \mathcal{G} }\:  \max_{D \in \mathcal{F}}\:  \mathbb{E}[D({\mathbf{X}})] - \mathbb{E}\bigl[ D(G(\mathbf{Z}))\bigr].
\end{equation} 
\end{ex}
Therefore, the moment-matching interpretation also holds for WGAN: for a convex set $\mathcal{F}$ of $1$-Lipschitz functions WGAN finds the generative model with  minimum $W_1$ distance to the distributions penalized to share the same moments over $\mathcal{F}$ with the data distribution.  We discuss two more examples in the Appendix: 1) for the indicator cost $c_{I}(\mathbf{x},\mathbf{x}')=\mathbb{I}(\mathbf{x}\neq \mathbf{x}')$ corresponding to the total variation distance we draw the connection to the energy-based GAN \cite{zhao2016energy}, 2) for the second-order cost $c_2(\mathbf{x},\mathbf{x}')={\Vert}\mathbf{x}-\mathbf{x}'{\Vert}^2$ we recover \cite{feizi2017understanding}'s quadratic GAN formulation under the LQG setting assumptions, i.e. linear generator, quadratic discriminator and Gaussian input data.
\section{Duality framework applied to neural net discriminators}
We applied the duality framework to analyze GAN problems with convex discriminator sets. However, a neural net set $\mathcal{F}_{nn} = \{ f_{\mathbf{w}}:\, \mathbf{w}\in \mathcal{W} \}$, where $f_{\mathbf{w}}$ denotes a neural net function with fixed architecture and weights $\mathbf{w}$ in feasible set $\mathcal{W}$, does not generally satisfy this convexity assumption. Note that a linear combination of several neural net functions in $\mathcal{F}_{nn}$ may not remain in $\mathcal{F}_{nn}$.

Therefore, we apply the duality framework to $\mathcal{F}_{nn}$'s convex hull, which we denote by  $\conv(\mathcal{F}_{nn})$, containing any convex combination of neural net functions in $\mathcal{F}_{nn}$. However, a convex combination of infinitely-many neural nets from $\mathcal{F}_{nn}$ is characterized by infinitely-many parameters, which makes optimizing the discriminator over $\conv(\mathcal{F}_{nn})$ computationally intractable. In the following theorem, we show that although a function in $\conv(\mathcal{F}_{nn})$ is a combination of infinitely-many neural nets, that function can be approximated by uniformly combining boundedly-many neural nets in $\mathcal{F}_{nn}$.
\begin{thm}\label{Thm: MIX-GAN}
Suppose any function $f_{\mathbf{w}}\in\mathcal{F}_{nn}$ is $L$-Lipschitz and bounded as $\vert f_{\mathbf{w}}(\mathbf{x})\vert \le M$. Also, assume that the $k$-dimensional random input $\mathbf{X}$ is norm-bounded as ${\Vert}\mathbf{X}{\Vert}_2\le R$. Then, any function in $\conv (\mathcal{F}_{nn})$ can be uniformly approximated over the ball ${\Vert}\mathbf{x}{\Vert}_2\le R$ within $\epsilon$-error by a uniform combination $\hat{f}(\mathbf{x}) = \frac{1}{m}\sum_{i=1}^m f_{\mathbf{w}_i}(\mathbf{x})$ of $m=\mathcal{O}(\frac{M^2k\log(LR/\epsilon)}{\epsilon^2})$ functions $(f_{\mathbf{w}_i})_{i=1}^m\in \mathcal{F}_{nn}$.
\end{thm}
\begin{proof}
We defer the proof to the Appendix. 
\end{proof}
The above theorem suggests using a uniform combination of multiple discriminator  nets to find a better approximation of the solution to the divergence minimization problem in Theorem \ref{Thm: dual max disc_general} solved over $\conv(\mathcal{F}_{nn})$. Note that this approach is different from MIX-GAN \cite{arora2017generalization} proposed for achieving equilibrium in GAN minimiax game. While our approach considers a uniform combination of multiple neural nets as the discriminator, MIX-GAN considers a randomized combination of the minimax game over multiple neural net discriminators and generators.

\section{Minimum-sum hybrid of f-divergence and Wasserstein distance: GAN with Lipschitz or adversarially-trained discriminator}
Here we apply the convex duality framework to a novel class of divergence measures. For each f-divergence $d_f$ we define divergence $d_{f,W_1}$, which is the minimum sum hybrid of $d_f$ and $W_1$ divergences, as follows  
\begin{equation}\label{minimum-sum distance}
d_{f,W_1}(P_1,P_2) := \inf_{Q}\: W_1(P_1,Q) + d_f(Q,P_2).
\end{equation}
The above infimum is taken over all distributions on random $\mathbf{X}$, searching for  distribution $Q$ minimizing the sum of the Wasserstein distance between $P_1$ and $Q$ and the f-divergence from $Q$ to $P_2$. Note that the hybrid of JS-divergence and $W_1$-distance defined earlier in \eqref{hybrid: JS, W1} is a special case of the above definition. 
While f-divergence in f-GAN does not change continuously with the generator parameters, the following theorem shows that similar to the continuous behavior of $W_1$-distance shown in \cite{arjovsky2017towards,arjovsky2017wasserstein} the proposed hybrid divergence changes continuously with the generative model. We defer the proofs of this section's results to the Appendix.
\begin{thm}\label{thm: cntinuity} 
Suppose $G_{\boldsymbol{\theta}}\in\mathcal{G}$ is continuously changing with parameters $\boldsymbol{\theta}$. Then, for any $Q$ and $\mathbf{Z}$, $d_{f,W_1}(P_{G_{\boldsymbol{\theta}}(\mathbf{Z})},Q)$ will behave continuously as a function of $\boldsymbol{\theta}$. Moreover, if $G_{\boldsymbol{\theta}}$ is assumed to be locally Lipschitz, then $d_{f,W_1}(P_{G_{\boldsymbol{\theta}}(\mathbf{Z})},Q)$ will be differentiable w.r.t. $\boldsymbol{\theta}$ almost everywhere. 
\end{thm}
Our next result reveals the minimax problem dual to minimizing this hybrid divergence with symmetric f-divergence component. We note that this symmetricity condition is met by the JS-divergence and the squared Hellinger divergence among the f-divergence examples discussed in \cite{nowozin2016f}.
\begin{thm}\label{Example: d_F W_1}
Consider $d_{f,W_1}$ with a symmetric f-divergence $d_f$, i.e. $d_f(P,Q)=d_f(Q,P)$, satisfying the assumptions in Theorem \ref{Theorem f_gan}.
If the composition $f^*\circ D$ is 1-Lipschitz for all $D\in\mathcal{F}$, the minimax problem in Theorem \ref{Thm: dual max disc_general} for the hybrid $d_{f,W_1}$ reduces to the f-GAN problem, i.e.
\begin{equation}\label{fW-GAN dual_firstorder}
\min_{ G\in \mathcal{G} }\;  \max_{D \in \mathcal{F}}\;  \mathbb{E}[D({\mathbf{X}})] - \mathbb{E}\bigl[\,  f^* \bigl( D(G(\mathbf{Z})\bigr) \,\bigr]. 
\end{equation}
\end{thm}
The above theorem reveals that when the Lipschitz constant of discriminator $D$ in f-GAN is properly regularized, then solving the f-GAN problem over the regularized discriminator also minimizes the continuous divergence measure $d_{f,W_1}$. As a special case, in the vanilla GAN problem \eqref{JS-GAN dual, GAN} we only need to constrain discriminator $\tilde{D}$ to be 1-Lipschitz, which can be done via the gradient penalty \cite{gulrajani2017improved} or spectral normalization of $\tilde{D}$'s weight matrices \cite{miyato2018spectral}, and then we minimize the continuously-behaving $d_{\JS , W_1}$. This result is also consistent with \cite{miyato2018spectral}'s empirical observations that regularizing the Lipschitz constant of the discriminator improves the training performance in vanilla GAN.

Our discussion has so far focused on the mixture of f-divergence and the first order Wasserstein distance, which suggests training f-GAN over Lipschitz-bounded discriminators. As a second solution, we prove that the desired continuity property can also be achieved through the following hybrid using the second-order Wasserstein ($W_2$) distance-squared:
\begin{equation}\label{minimum-sum distance, W2}
d_{f,W_2}(P_1,P_2) := \inf_{Q}\: W^2_2(P_1,Q) + d_f(Q,P_2).
\end{equation}
\begin{thm}\label{thm: cntinuity, W2} 
Suppose $G_{\boldsymbol{\theta}}\in\mathcal{G}$ continuously changes with parameters $\boldsymbol{\theta}\in\mathbb{R}^k$. Then, for any distribution $Q$ and random vector $\mathbf{Z}$, $d_{f,W_2}(P_{G_{\boldsymbol{\theta}}(\mathbf{Z})},Q)$ will be continuous in $\boldsymbol{\theta}$. Also, if we further assume $G_{\boldsymbol{\theta}}$ is bounded and locally-Lipschitz w.r.t. $\boldsymbol{\theta}$, then the hybrid divergence $d_{f,W_2}(P_{G_{\boldsymbol{\theta}}(\mathbf{Z})},Q)$ is almost everywhere differentiable w.r.t. $\boldsymbol{\theta}$.
\end{thm}
The following result shows that minimizing $d_{f,W_2}$ reduces to f-GAN problem where the discriminator is being adversarially trained. 
\begin{thm}\label{thm: d_f,c min}
Assume $d_f$ and $\mathcal{F}$ satisfy the assumptions in Theorem \ref{Example: d_F W_1}. Then, the minimax problem in Theorem \ref{Thm: dual max disc_general} corresponding to the hybrid $d_{f,W_2}$ divergence reduces to
\begin{align}\label{fW-GAN dual}
\min_{ G\in \mathcal{G} }\;  \max_{D \in \mathcal{F}}\;  \mathbb{E}[D({\mathbf{X}})] + \mathbb{E}\bigl[\, \min_{\mathbf{u}}\, -f^* \bigl(D( \, G(\mathbf{Z}) + \mathbf{u}\,)\bigr) + \Vert  \mathbf{u} \Vert^2  \,\bigr].
\end{align}
\end{thm}
The above result reduces minimizing the hybrid $d_{f,W_2}$ divergence to an f-GAN minimax game with a new third player. Here the third player assists the generator by perturbing the generated fake samples in order to make them harder to be distinguished from the real samples by the discriminator. The cost for perturbing a fake sample $G(\mathbf{Z})$ to $G(\mathbf{Z})+\mathbf{u}$ will be $\Vert \mathbf{u}\Vert^2$, which constrains the power of the third player who can be interpreted as an adversary to the discriminator. To implement the game between these three players, we can adversarially learn the discriminator while we are training GAN, using the Wasserstein risk minimization (WRM) adversarial learning scheme discussed in \cite{sinha2018certifiable}.

\section{Numerical Experiments}
To evaluate our theoretical results, we used the CelebA \cite{liu2015faceattributes} and LSUN-bedroom \cite{yu2015lsun} datasets. Furthermore, in the Appendix we include the results of our experiments over the MNIST \cite{lecun1998mnist} dataset. We considered vanilla GAN \cite{goodfellow2014generative} with the minimax formulation in \eqref{JS-GAN dual, GAN} and DCGAN \cite{radford2015unsupervised} convolutional architecture for discriminator and generator. We used the code provided by \cite{gulrajani2017improved} and trained DCGAN via Adam optimizer \cite{kingma2014adam} for 200,000 generator iterations. We applied 5 discriminator updates for each generator update.

\begin{figure}[h]
\centering 
  \includegraphics[width=.975\linewidth]{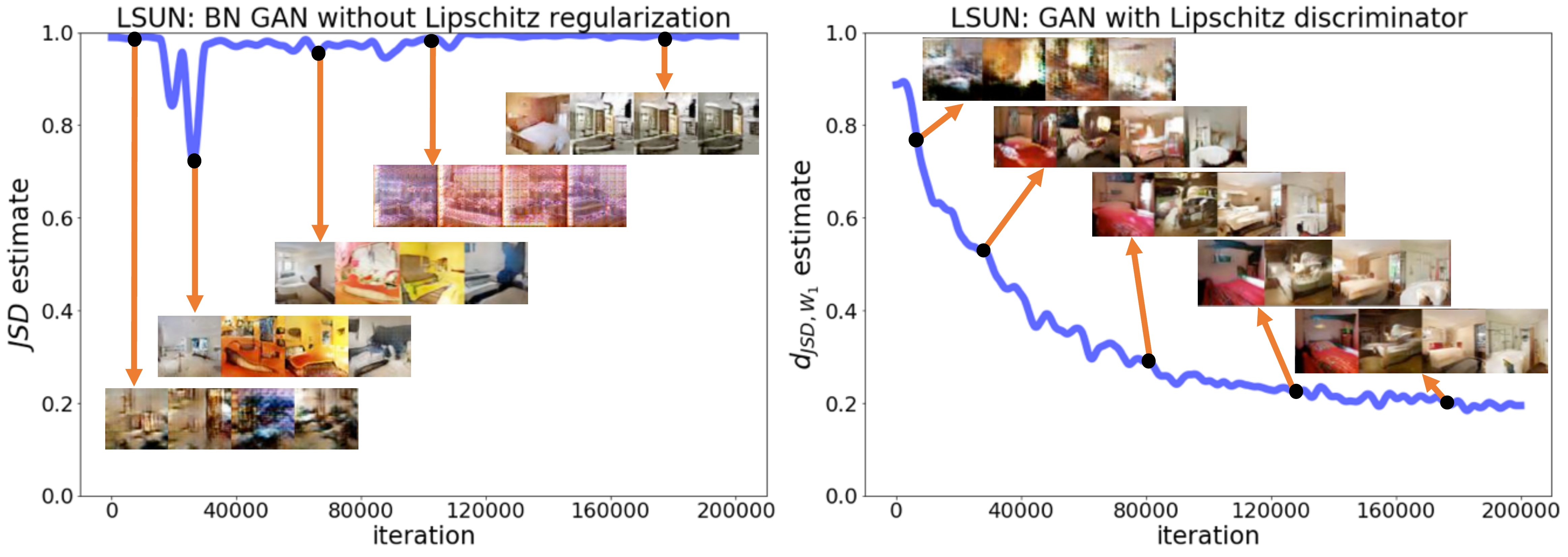} 
  \captionof{figure}{Divergence estimate in DCGAN trained over LSUN samples, (left) JS-divergence in standard DCGAN regularized with batch normalization, (right) hybrid $d_{\JS , W_1}$ in DCGAN with 1-Lipschitz discriminator regularized via spectral normalization. }  
  \label{lsun-fig}
\end{figure}
Figure \ref{lsun-fig} shows how the discriminator loss evaluated over 2000 validation samples, which is an estimate of the divergence measure, changes as we train the DCGAN over LSUN samples. Using standard DCGAN regularizied by only batch normalization (BN) \cite{ioffe2015batch}, we observed (Figure \ref{lsun-fig}-left) that the JS-divergence estimate always remains close to its maximum value $1$ and also poorly correlates with the visual quality of  generated samples. In this experiment, the GAN training failed and led to mode collapse starting at about the 110,000th iteration. On the other hand, after replacing BN with spectral normalization (SN) \cite{miyato2018spectral} to ensure the discriminator's Lipschitzness, the discriminator loss decreased in a desired monotonic fashion (Figure \ref{lsun-fig}-right). This observation is consistent with Theorems \ref{thm: cntinuity} and \ref{Example: d_F W_1} showing that the discriminator loss becomes an estimate for the hybrid $d_{\JS ,W_1}$ divergence changing continuously with the generator parameters. Also, the samples generated by the Lipschitz-regularized DCGAN looked qualitatively better and correlated well with the estimate of $d_{\JS ,W_1}$ divergence.

Figure \ref{celeba-fig} shows the results of similar experiments over the CelebA dataset. Again, we observed (Figure \ref{celeba-fig}-top left) that the JS-divergence estimate remains close to $1$ while training DCGAN with BN. However, after applying two different Lipschitz regularization methods, SN and the gradient penalty (GP) \cite{gulrajani2017improved} in Figures \ref{celeba-fig}-top right and bottom left,  we observed that the hybrid $d_{\JS , W_1}$ changed nicely and monotonically, and correlated properly with the sharpness of samples generated. Figure \ref{celeba-fig}-bottom right shows that a similar desired behavior can also be achieved using the second-order hybrid $d_{\JS , W_2}$ divergence. In this case, we trained the DCGAN discriminator via the WRM adversarial learning scheme \cite{sinha2018certifiable}.
\begin{figure}[t]
\centering 
  \includegraphics[width=1.\linewidth]{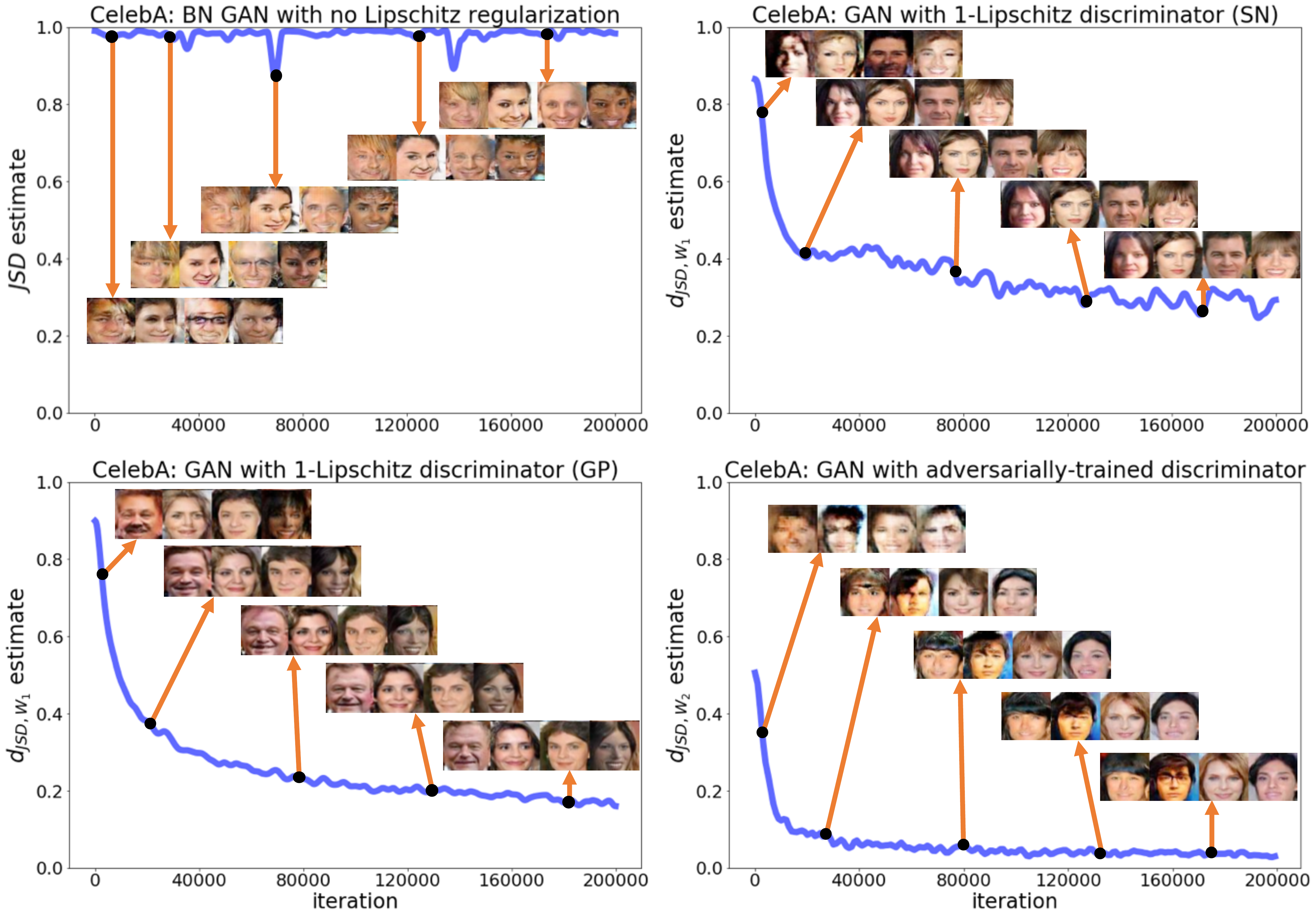} 
  \captionof{figure}{Divergence estimate in DCGAN trained over CelebA samples, (top-left) JS-divergence in DCGAN regularized with batch normalization, (top-right) hybrid $d_{\JS , W_1}$ in DCGAN with spectrally-normalized discriminator, (bottom-left) hybrid $d_{\JS , W_1}$ in DCGAN with 1-Lipschitz discriminator regularized via the gradient penalty, (bottom-right) hybrid $d_{\JS , W_2}$ in DCGAN with discriminator being adversarially-trained using WRM.} 
  \label{celeba-fig}
\end{figure}
\vspace*{-1.5mm}
\section{Related Work}
Theoretical studies of GAN have focused on three different aspects: approximation, generalization, and optimization. On the approximation properties of GAN, \cite{liu2017approximation} studies GAN's approximation power using a moment-matching approach. The authors view the maximized discriminator objective as an $\mathcal{F}$-based adversarial divergence, showing that the adversarial divergence between two distributions takes its minimum value if and only if the two distributions share the same moments over $\mathcal{F}$. Our convex duality framework interprets their result and further draws the connection to the original divergence measure. \cite{nock2017f} studies the f-GAN problem through an information geometric approach based on the Bregman divergence and its connection to f-divergence. 

Analyzing GAN's generalization performance is another problem of interest in several recent works. \cite{arora2017generalization} proves generalization guarantees for GANs in terms of $\mathcal{F}$-based distance measures. \cite{arora2017gans} uses an elegant approach based on the Birthday Paradox to empirically study the generalizibility of GAN's learned models.  \cite{santurkar2017classification} develops a quantitative approach for examining diversity and generalization in GAN's  learned distribution. \cite{zhang2018on} studies approximation-generalization trade-offs in GAN by analyzing the discriminative power of  $\mathcal{F}$-based distances. Regarding optimization properties of GAN, \cite{chen2018training,zhao2018information} propose duality-based methods for improving the optimization performance in training deep generative models. \cite{roth2017stabilizing} suggests applying noise convolution with input data for boosting the training performance in f-GAN. Moreover, several other works including \cite{nagarajan2017gradient,mescheder2017numerics,daskalakis2017training,
feizi2017understanding,sanjabi2018solving} 
 explore the optimization and stability properties of training GANs. Finally, we note that the same convex analysis approach used in this paper has further provided a powerful theoretical framework to analyze various supervised and unsupervised learning problems \cite{dudik2007maximum,razaviyayn2015discrete, farnia2016minimax,fathony2016adversarial,fathony2017adversarial}.

{\noindent \textbf{Acknowledgments:}} We are grateful for support under a  Stanford Graduate Fellowship, the National Science Foundation grant under CCF-1563098, and the Center for Science of Information (CSoI), an NSF Science and Technology Center under grant agreement CCF-0939370.

\newpage \small
\bibliographystyle{unsrt}

 \normalsize

\section{Appendix}
\subsection{Additional numerical results}
\subsubsection{LSUN divergence estimates for different training schemes}
Figure \ref{lsun_complete} shows the complete divergence estimates over LSUN dataset for the GAN training schemes described in the main text. While the hybrid divergence measures $d_{\JS ,W_1}$, $d_{\JS , W_2}$ decreased smoothly as the DCGAN was being trained, the JS-divergence always remained close to its maximum value $1$ which led to lower-quality produced samples.
\begin{figure}[h]
\centering 
  \includegraphics[width=1.\linewidth]{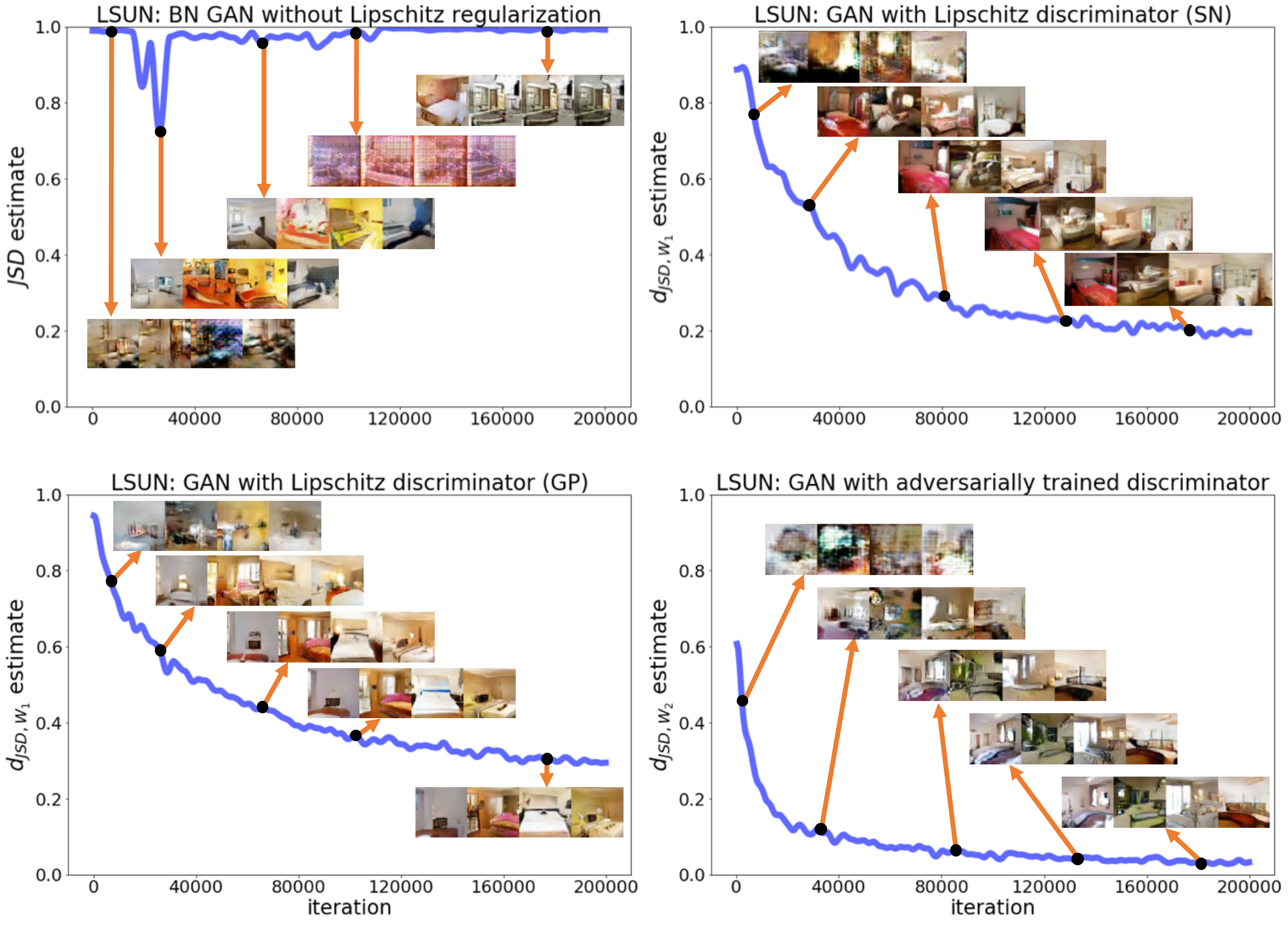} 
  \captionof{figure}{Divergence estimate in DCGAN trained over LSUN samples, (top-left) JS-divergence in DCGAN regularized with batch normalization, (top-right) hybrid $d_{\JS , W_1}$ in DCGAN with spectrally-normalized discriminator, (bottom-left) hybrid $d_{\JS , W_1}$ in DCGAN with 1-Lipschitz discriminator regularized via the gradient penalty, (bottom-right) hybrid $d_{\JS , W_2}$ in DCGAN with discriminator being adversarially-trained using WRM.} 
  \label{lsun_complete}
\end{figure}

\subsubsection{CelebA, LSUN, MNIST images generated by different trainings of DCGAN}
Figures \ref{celeba-samples-fig}, \ref{lsun-samples-fig}, and \ref{mnist-samples-fig} show the CelebA, LSUN, and MNIST samples generated by vanilla DCGAN trained via the different methods described in the main text. Observe that applying Lipschitz regularization and  adversarial training to the discriminator consistently result in the highest quality generator output samples. We note that tight SN in these figures refers to \cite{tsuzuku2018lipschitz}'s spectral normalization method for convolutional layers, which precisely normalizes a conv layer's spectral norm and hence guarantees the $1$-Lipschitzness of the discriminator neural net. Note that for non-tight SN we use the original heuristic for normalizing convolutional layers' operator norm introduced in \cite{miyato2018spectral}.
\begin{figure}[h]
\centering 
  \includegraphics[width=.975\linewidth]{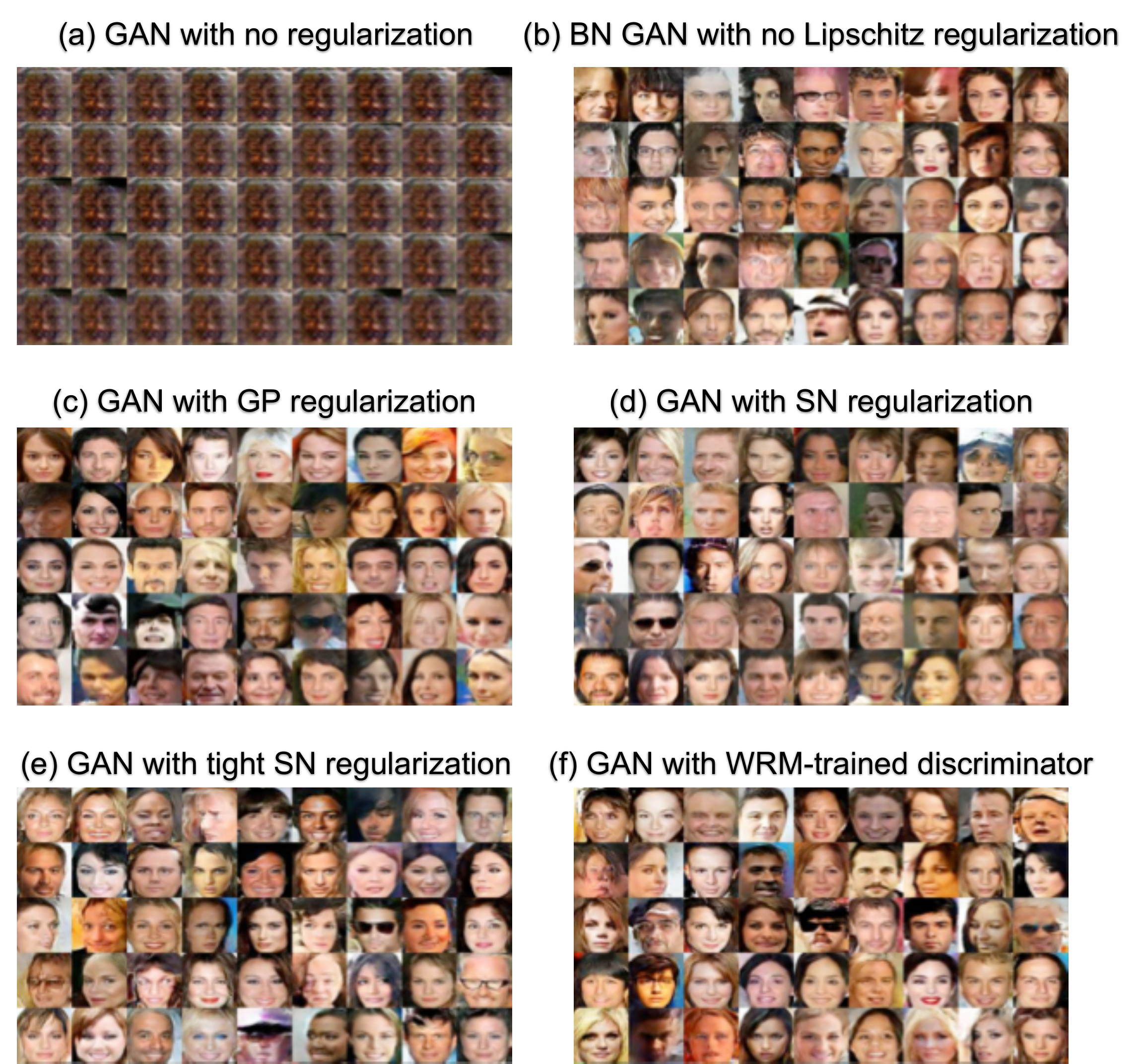}
  \captionof{figure}{Samples generated by DCGAN trained over CelebA samples}  \label{celeba-samples-fig}
\end{figure}
\begin{figure}
\centering 
  \includegraphics[width=.975\linewidth]{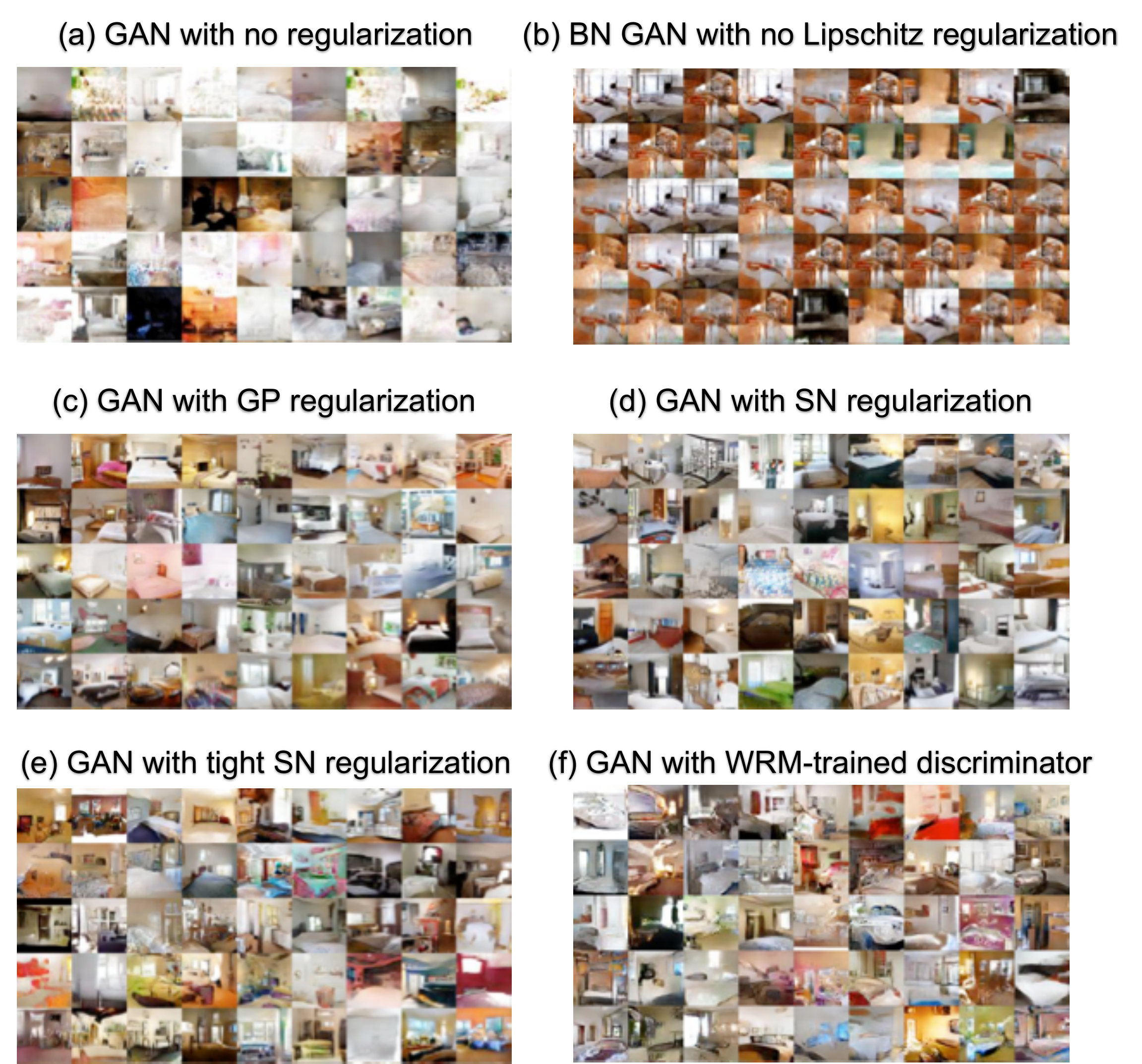}
  \captionof{figure}{Samples generated by DCGAN trained over LSUN-bedroom samples}  \label{lsun-samples-fig}
\end{figure}
\begin{figure}
\centering 
  \includegraphics[width=.975\linewidth]{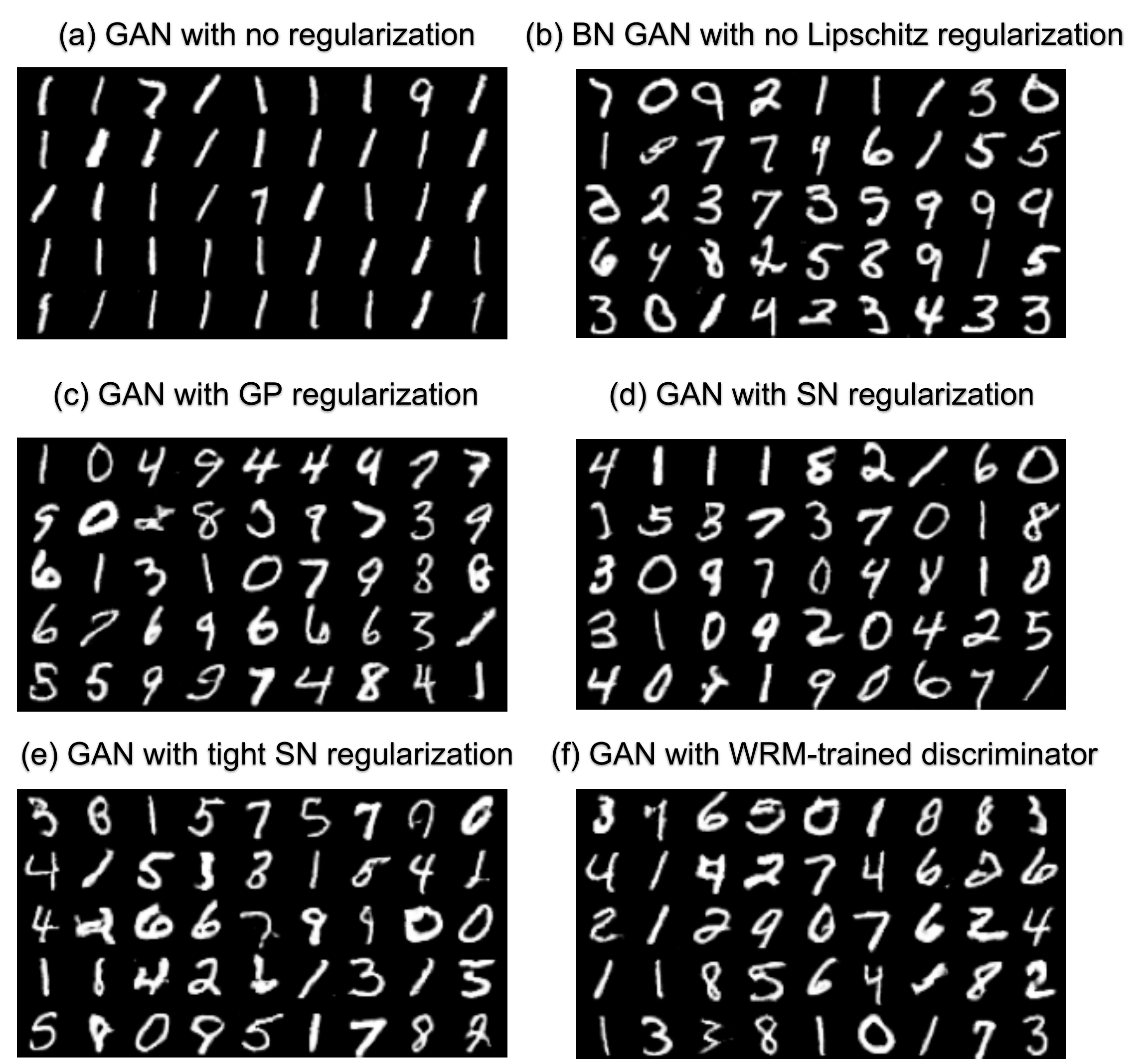}
  \captionof{figure}{Samples generated by DCGAN trained over MNIST samples}  \label{mnist-samples-fig}
\end{figure}
\subsection{Proof of Theorem 1}
Theorem 1 and Corollary 1 directly result from the following two lemmas.
\begin{lemma} \label{Lemma: dual max disc_general}
Suppose divergence $d(P,Q)$ is non-negative, lower semicontinuous and convex in distribution $Q$. Consider a convex subset of continuous functions $\mathcal{F}$ and assume support set $\mathcal{X}$ is compact. Then, the following duality holds for any pair of distributions $P_1,P_2$:
\begin{equation}\label{Eq: dual max disc}
 \max_{D \in \mathcal{F}} \,    \mathbb{E}_{P_2}[D(\mathbf{X})] -  d^*_{P_1}(D)  \, = \, \min_{Q}\:\bigl\{\, d( P_1 , Q) + \max_{D\in\mathcal{F}}\,\{ \,\mathbb{E}_{P_2}[D(\mathbf{X})] - \mathbb{E}_{Q}[D(\mathbf{X})]\,  \}\,\bigr\}.
\end{equation}
\end{lemma}
\begin{proof}
Note that
\begin{align}
  &\quad \min_{Q}\:\bigl\{\, d( P_1 , Q) + \max_{D\in\mathcal{F}}\,\{ \,\mathbb{E}_{P_2}[D(\mathbf{X})] - \mathbb{E}_{Q}[D(\mathbf{X})]\,  \}\,\bigr\} \, \nonumber \\
   &=  \, \min_{Q}\; \max_{D\in\mathcal{F}}\,\bigl\{\, d( P_1 , Q) + \mathbb{E}_{P_2}[D(\mathbf{X})] - \mathbb{E}_{Q}[D(\mathbf{X})]\,  \bigr\} \nonumber \\
& \stackrel{(a)}{=}  \, \max_{D\in\mathcal{F}}\;\min_{Q}\; \,\bigl\{\, d( P_1 , Q) + \mathbb{E}_{P_2}[D(\mathbf{X})] - \mathbb{E}_{Q}[D(\mathbf{X})]\,  \bigr\}   \\
& \stackrel{}{=}  \, \max_{D\in\mathcal{F}}\; \bigl\{\, \mathbb{E}_{P_2}[D(\mathbf{X})] + \min_{Q}\,\{ d( P_1 , Q) - \mathbb{E}_{Q}[D(\mathbf{X})]\, \}\, \bigr\}  \nonumber  \\
& \stackrel{}{=}  \, \max_{D\in\mathcal{F}}\; \bigl\{\, \mathbb{E}_{P_2}[D(\mathbf{X})] - \max_{Q}\,\{ \mathbb{E}_{Q}[D(\mathbf{X})]-d( P_1 , Q)  \, \}\, \bigr\}  \nonumber  \\
& \stackrel{(b)}{=}  \, \max_{D\in\mathcal{F}}\;  \mathbb{E}_{P_2}[D(\mathbf{X})] -d^*_{P_1}(D).  \nonumber
\end{align}
Here (a) is a consequence of the generalized Sion's minimax theorem \cite{borwein2016very}, because the space of probability measures on compact $\mathcal{X}$ is convex and weakly compact \cite{billingsley2013convergence}, $\mathcal{F}$ is assumed to be convex, the minimiax objective is lower semicontinuous and convex in $Q$ and linear in $D$. (b) holds according to the conjugate $d^*_P$'s definition. 
\end{proof}


\begin{lemma}
Assume divergence $d(P,Q)$ is non-negative, lower semicontinuous and convex in distribution $Q$ over compact $\mathcal{X}$. Consider a linear space subset of continuous functions $\mathcal{F}$. Then, the following duality holds for any pair of distributions $P_1,P_2$:
\begin{equation}\label{Eq: dual max disc_linear}
  \min_{Q\in \mathcal{P}_\mathcal{F}(P_2)}\; d( P_1 , Q) \, = \, \max_{D \in \mathcal{F}} \,    \mathbb{E}_{P_2}[D(\mathbf{X})] - d^*_{P_1}(D).
\end{equation}
\end{lemma}
\begin{proof}
This lemma is a consequence of Lemma 1. Note that a linear space $\mathcal{F}$ is a convex set. Therefore, Lemma 1 applies to $\mathcal{F}$. However, since $\mathcal{F}$ is a linear space i.e. for any $D\in\mathcal{F}$ and $\lambda\in\mathbb{R}$ it includes $\lambda D$ we have 
\begin{align}
\max_{D\in\mathcal{F}}\,\{ \,\mathbb{E}_{P_2}[D(\mathbf{X})] - \mathbb{E}_{Q}[D(\mathbf{X})]\,  \} \, = \, \begin{cases}
\begin{aligned}
0 \quad &\text{\rm if}\; Q\in\mathcal{P_F}(P_2) \\
+\infty \;\; &\text{\rm otherwise.}
\end{aligned}
\end{cases} 
\end{align}
As a result, the minimizing $Q^*$ precisely matches the moments over $\mathcal{F}$ to $P_2$'s moments, which completes the proof.
\end{proof}
\subsection{Proof of Theorem 2}
We first prove the following lemma.
\begin{lemma}
Consider f-divergence $d_f$ corresponding to function $f$ which has a non-decreasing convex-conjugate $f^*$. Then, for any continuous $D$
\begin{equation}\label{f-divergence conjugate}
{d_f}^*_{P}(D) =  \mathbb{E}_P\bigl[ f^*\bigl(\, D(\mathbf{X})+\lambda_0\,\bigr) \bigr] - \lambda_0
\end{equation}
where $\lambda_0\in\mathbb{R}$ satisfies $\mathbb{E}_P\bigl[ {f^*}'\, \bigl(\, D(\mathbf{X})+\lambda_0\,\bigr) \bigr] = 1$. Here ${f^*}'$ stands for the derivative of conjugate function $f^*$ which is supposed to be non-negative everywhere.
\end{lemma}
\begin{proof}
Note that
\begin{align}
{d_f}^*_{P}(D) &\stackrel{(a)}{=} \sup_Q\: \mathbb{E}_Q[D(\mathbf{X})] - {d_f}(P,Q) \nonumber \\
& \stackrel{(b)}{=} \sup_Q\: 
\mathbb{E}_Q[D(\mathbf{X})] - \mathbb{E}_P\bigl[ f\bigl(\frac{q(\mathbf{X})}{p(\mathbf{X})}\bigr) \bigr] \nonumber \\
&\stackrel{(c)}{=} \max_{
q(\mathbf{x})\ge 0,\, \int q(\mathbf{x})\mathop{d \mathbf{x}} =1
}\: 
  \int q(\mathbf{x}) D(\mathbf{x}) \mathop{d \mathbf{x}}- \mathbb{E}_P\bigl[\, f\bigl(\frac{q(\mathbf{X})}{p(\mathbf{X})}\bigr)\, \bigr] \nonumber \\
& \stackrel{(d)}{=} \min_{\lambda\in\mathbb{R}}\: -\lambda + \max_{q(\mathbf{x})\ge 0}\: 
\int q(\mathbf{x})\bigl( D(\mathbf{x})+\lambda \bigr)\mathop{d \mathbf{x}} -\mathbb{E}_P\bigl[\, f\bigl(\frac{q(\mathbf{X})}{p(\mathbf{X})}\bigr)\, \bigr] \nonumber \\
& \stackrel{(e)}{=} \min_{\lambda\in\mathbb{R}}\: - \lambda + \max_{r(\mathbf{x})\ge 0}\: 
\mathbb{E}_P\bigl[\,   r(\mathbf{X})\bigl( D(\mathbf{X})+\lambda \bigr) - f(r(\mathbf{X}))\, 
\, \bigr] \nonumber \\
& \stackrel{(f)}{=} \min_{\lambda\in\mathbb{R}}\: -\lambda +
\mathbb{E}_P\bigl[\, \max_{r(\mathbf{X})\ge 0}\:  r(\mathbf{X})\bigl( D(\mathbf{X})+\lambda \bigr) - f(r(\mathbf{X}))
\, \bigr] \nonumber \\
& \stackrel{(g)}{=} \min_{\lambda\in\mathbb{R}}\: -\lambda +
\mathbb{E}_P\bigl[\, f^*\bigl( D(\mathbf{X})+\lambda \bigr)
\, \bigr] \nonumber \\
& \stackrel{}{=} - \max_{\lambda\in\mathbb{R}}\: \lambda -
\mathbb{E}_P\bigl[\, f^*\bigl( D(\mathbf{X})+\lambda \bigr)
\, \bigr] \label{df-conjuagte-max} \\
& \stackrel{(h)}{=} - \lambda_0 +
\mathbb{E}_P\bigl[\, f^*\bigl( D(\mathbf{X})+\lambda_0 \bigr)
\, \bigr]. \label{df-conjuagte}
\end{align}
Here (a) and (b) follow from the conjugate $d^*_P$ and f-divergence $d_f$ definitions. (c) rewrites the optimization problem in terms of the density function $q$ corresponding to distribution $Q$. (d) uses the strong convex duality to move the density constraint $\int q(\mathbf{x})\mathop{d \mathbf{x}} =1$ to the objective. Note that strong duality holds, since we have a convex optimization problem with affine constraints. (e) rewrites the problem after a change of variable $r(\mathbf{x})=q(\mathbf{x})/p(\mathbf{x})$. (f) holds since $f$ and $D$ are assumed to be continuous. (g) follows from the assumption that the derivative of $f^*$ takes non-negative values, and hence the minimizing $r(\mathbf{x})\ge 0$ also minimizes the unconstrained optimization for the convex conjugate $f^*$
\begin{equation*}
f^*\bigl(D(\mathbf{X})+\lambda \bigr) := \max_{r(\mathbf{X})}\:  r(\mathbf{X})\bigl( D(\mathbf{X})+\lambda \bigr) - f(r(\mathbf{X})).
\end{equation*}
Taking the derivative of the concave objective, the $\lambda$ value maximizing the objective solves the equation $\mathbb{E}_P\bigl[ {f^*}'\, \bigl(\, D(\mathbf{X})+\lambda\,\bigr) \bigr] = 1$ which is assumed to be $\lambda_0$. Therefore, (h) holds and the proof is complete.
\end{proof} 
Now we prove Theorem 2 which can be broken into two parts as follows.
\begin{thm*}[Theorem 2]
Consider f-divergence $d_f$ where $f$ has a non-decreasing conjugate $f^*$. \\ (a) Suppose $\mathcal{F}$ is a convex set closed to a constant addition, i.e. for any $D\in\mathcal{F},\, \lambda\in\mathbb{R}$ we have $D+\lambda\in\mathcal{F}$. Then,
\begin{align}\label{f-GAN dual Appendix}
  & \min_{P_{G(\mathbf{Z})}\in\mathcal{P_G}}\; \min_{Q_\mathbf{X}}\; d_f( P_{G(\mathbf{Z})} , Q) + \max_{D\in\mathcal{F}}\,\bigl\{ \mathbb{E}_{P_\mathbf{X}}[D(\mathbf{X})] - \mathbb{E}_{Q}[D(\mathbf{X})]  \bigr\} \nonumber \\
 = &\;\;\;\min_{ G\in \mathcal{G} }\:  \max_{D \in \mathcal{F}}\:  \mathbb{E}_{P_\mathbf{X}}[D({\mathbf{X}})] - \mathbb{E}\bigl[f^* \bigl( D(G(\mathbf{Z})) \bigr)\bigr]. 
\end{align}
(b) Suppose $\mathcal{F}$ is a linear space including the constant function $D_0(\mathbf{x})=1$. Then,
\begin{equation}\label{f-GAN dual_linear Appendix}
\min_{P_{G(\mathbf{Z})}\in\mathcal{P_G}}\; \min_{Q_\mathbf{X}\in \mathcal{P}_\mathcal{F}(P_\mathbf{X})} \; d_f( P_{G(\mathcal{Z})} , Q) \, = \, \min_{ G\in \mathcal{G} }\:  \max_{D \in \mathcal{F}}\:  \mathbb{E}_{P_\mathbf{X}}[D({\mathbf{X}})] - \mathbb{E}\bigl[f^* \bigl( D(G(\mathbf{Z})) \bigr)\bigr] .
\end{equation}
\end{thm*}
\begin{proof}
This theorem is an application of Theorem 1 and Corollary 1. For part (a) we have
\begin{align*}
  & \min_{P_{G(\mathbf{Z})}\in\mathcal{P_G}}\; \min_{Q_\mathbf{X}}\; d_f( P_{G(\mathbf{Z})} , Q) + \max_{D\in\mathcal{F}}\,\bigl\{ \mathbb{E}_{P_\mathbf{X}}[D(\mathbf{X})] - \mathbb{E}_{Q}[D(\mathbf{X})]  \bigr\} \nonumber \\   
\stackrel{(c)}{=} &\;\;\;\min_{ G\in \mathcal{G} }\:  \max_{D \in \mathcal{F}}\:  \mathbb{E}_{P_\mathbf{X}}[D({\mathbf{X}})] - {d_f}^*_{P_{G(\mathbf{Z})}}(D) \\   
 \stackrel{(d)}{=} &\;\;\;\min_{ G\in \mathcal{G} }\:  \max_{D \in \mathcal{F}}\:  \mathbb{E}_{P_\mathbf{X}}[D({\mathbf{X}})] + \max_{\lambda\in\mathbb{R}} \lambda - \mathbb{E}\bigl[ f^*\bigl(\, D(G(\mathbf{Z}))+\lambda \, \bigr)\bigr]  \\  
  = &\;\;\;\min_{ G\in \mathcal{G} }\:  \max_{D \in \mathcal{F}, \lambda\in\mathbb{R}}\:  \mathbb{E}_{P_\mathbf{X}}[D({\mathbf{X}}) + \lambda ]- \mathbb{E}\bigl[ f^*\bigl(\, D(G(\mathbf{Z}))+\lambda \, \bigr)\bigr]  \\   
  \stackrel{(e)}{=} &\;\;\;\min_{ G\in \mathcal{G} }\:  \max_{D \in \mathcal{F}}\:  \mathbb{E}_{P_\mathbf{X}}[D({\mathbf{X}})] - \mathbb{E}\bigl[f^* \bigl( D(G(\mathbf{Z})) \bigr)\bigr]. 
\end{align*}
Here (c) is a direct result of Theorem 1. (d) uses the simplified version \eqref{df-conjuagte-max} for ${d_f}^*_P$. (e) follows from the assumption that $\mathcal{F}$ is closed to constant additions.

For part (b) note that since $\mathcal{F}$ is a linear space and includes $D_0(\mathbf{x})=1$, it is closed to constant additions. Hence, an application of Corollary 1 reveals
\begin{align*}
\min_{P_{G(\mathbf{Z})}\in\mathcal{P_G}}\; \min_{Q_\mathbf{X}\in \mathcal{P}_\mathcal{F}(P_\mathbf{X})} \; d_f( P_{G(\mathcal{Z})} , Q) \, &= \, \min_{ G\in \mathcal{G} }\:  \max_{D \in \mathcal{F}}\:  \mathbb{E}_{P_\mathbf{X}}[D({\mathbf{X}})] - {d_f}^*_{P_{G(\mathbf{Z})}}(D) \\
&=  \min_{ G\in \mathcal{G} }  \max_{D \in \mathcal{F}}  \mathbb{E}_{P_\mathbf{X}}[D({\mathbf{X}})] + \max_{\lambda\in\mathbb{R}} \lambda - \mathbb{E}\bigl[ f^*\bigl(\, D(G(\mathbf{Z}))+\lambda \, \bigr)\bigr] \\
&= \, \min_{ G\in \mathcal{G} }\:  \max_{D \in \mathcal{F},\lambda\in\mathbb{R}}\:  \mathbb{E}_{P_\mathbf{X}}[D({\mathbf{X}})+\lambda] - \mathbb{E}\bigl[f^* \bigl( D(G(\mathbf{Z}))+\lambda \bigr)\bigr] \\
&= \, \min_{ G\in \mathcal{G} }\:  \max_{D \in \mathcal{F}}\:  \mathbb{E}_{P_\mathbf{X}}[D({\mathbf{X}})] - \mathbb{E}\bigl[f^* \bigl( D(G(\mathbf{Z})) \bigr)\bigr],
\end{align*}
which makes the proof complete.
\end{proof}

\subsection{Proof of Theorem 3}
Theorem 3 is a direct application of the following lemma to Theorem 1 and Corollary 1.
\begin{lemma}
Let $c$ be a lower semicontinuous non-negative cost function. Considering the c-transform operation $D^c$ defined in the text, the following holds for any continuous $D$
\begin{equation}\label{Conjugate: Wc}
{OT_c}^*_P(D) = \mathbb{E}_P[\,D^c(\mathbf{X})\, ].
\end{equation}
\end{lemma}
\begin{proof}
We have
\begin{align*}
{OT_c}^*_P(D) &\stackrel{(a)}{=} \sup_{Q}\:  \mathbb{E}_Q[D(\mathbf{X}')] - OT_c(P,Q)\\
& \stackrel{(b)}{=} - \inf_{Q}\: \inf_{M\in\Pi(P,Q)} \mathbb{E}_M\bigl[\, c(\mathbf{X},\mathbf{X}') - D(\mathbf{X}')\,\bigr] \\
& \stackrel{}{=} - \inf_{Q, M\in\Pi(P,Q)} \mathbb{E}_M\bigl[\, c(\mathbf{X},\mathbf{X}') - D(\mathbf{X}')\,\bigr] \\
&\stackrel{(c)}{\ge} - \mathbb{E}_P\bigl[\,\inf_{\mathbf{x}'}\: c(\mathbf{X},\mathbf{x}') - D(\mathbf{x}')\,\bigr] \\
&\stackrel{}{=} \mathbb{E}_P\bigl[\,\sup_{\mathbf{x}'}\: D(\mathbf{x}') - c(\mathbf{X},\mathbf{x}') \,\bigr] \\
& \stackrel{(d)}{=} \mathbb{E}_P [D^c(\mathbf{X})]. 
\end{align*}
Here (a), (b), (d) hold according to the definitions. Moreover, we show (c) will hold with equality under the lemma's assumptions. $c(\mathbf{x},\mathbf{x}') - D(\mathbf{x}')$ is lower semicontinuous, and hence for every $\epsilon >0$ there exists a measurable function $v(\mathbf{x)}$ such that for the coupling $M=\pi_{\mathbf{X},v(\mathbf{X})}$  the absolute difference $\bigl| \mathbb{E}_M\bigl[\, c(\mathbf{X},\mathbf{X}') - D(\mathbf{X}')\,\bigr] -  \mathbb{E}_P\bigl[\,\inf_{\mathbf{x}'}\: c(\mathbf{X},\mathbf{x}') - D(\mathbf{x}')\,\bigr]\bigr|<\epsilon$ is $\epsilon$-bounded. Therefore, $(c)$ holds with equality and the proof is complete.
\end{proof}
\subsection{Proof of Theorem 4}
Consider a convex combination of functions from $\mathcal{F}_{nn}$ as $f_{\alpha}(\mathbf{x}) = \int \alpha(\mathbf{w})f_{\mathbf{w}}(\mathbf{x})\mathop{d\mathbf{w} } $ where $\alpha$ can be considered as a probability density function over feasible set $\mathcal{W}$. 
 Consider $m$ samples $(\mathbf{W}_i)_{i=1}^m$ taken i.i.d. from $\alpha$. Since any $f_\mathbf{w}$ is $M$-bounded, according to Hoeffding's inequality for a fixed $\mathbf{x}$ we have
\begin{equation}
\Pr \biggl( \, \biggl| \, \frac{1}{m}\sum_{i=1}^m f_{\mathbf{W}_i}(\mathbf{x}) \, -  \, \mathbb{E}_{\mathbf{W}\sim \alpha}\bigl[f_\mathbf{W}(\mathbf{x})\bigr]  \biggr| \ge \frac{\epsilon}{2}\, \biggr)\, \le \, 2 \, \exp \bigl(-\frac{m\epsilon^2}{8M^2}\bigr).
\end{equation}
Next we consider a $\delta$-covering for the ball $\{\mathbf{x}:\, ||\mathbf{x}||_2\le R \}$, where we choose $\delta = \frac{\epsilon}{4L}$. We know a $\delta$-covering $\{\mathbf{x}_j: 1\le j \le N\}$ exists with a bounded size $N \le (12LR/\epsilon)^k$ \cite{csiszar2004information}. Then, an application of the union bound implies
\begin{align*}
 \Pr \biggl(  \max_{1\le j\le N}\, \biggl| \, \frac{1}{m}\sum_{i=1}^m f_{\mathbf{W}_i}(\mathbf{x}_j)  -  \mathbb{E}_{\mathbf{W}\sim \alpha}\bigl[f_\mathbf{W}(\mathbf{x}_j)\bigr]  \biggr| \ge \frac{\epsilon}{2}\biggr)
 &\le 2  N   \exp \bigl(-\frac{m\epsilon^2}{8M^2}\bigr) \\
& \le \exp\biggl( -\frac{m\epsilon^2}{8M^2} + k\log\bigl(\frac{12LR}{\epsilon}\bigr) +\log 2 \biggr)
 \end{align*}
Hence if we have $ -\frac{m\epsilon^2}{8M^2} + k\log(\frac{12LR}{\epsilon}) +\log 2 < 0$ the above upper-bound is strictly less than $1$, showing there exists at least one outcome $(\mathbf{w}_i)_{i=1}^m$ satisfying
 \begin{equation}\label{Thm 3 proof_ eq 1}
 \max_{1\le j\le N}\, \biggl| \, \frac{1}{m}\sum_{i=1}^m f_{\mathbf{w}_i}(\mathbf{x}_j)  -   \mathbb{E}_{\mathbf{W}\sim \alpha}\bigl[f_\mathbf{W}(\mathbf{x}_j)\bigr]  \biggr| < \frac{\epsilon}{2}.
 \end{equation}
 Then, we claim the following holds over the norm-bounded $\{\mathbf{x}:\: ||\mathbf{x}||_2\le R \}$:
  \begin{equation} \label{Thm 3 proof_ to show}
 \sup_{||\mathbf{x}||_2\le R}\, \biggl| \, \frac{1}{m}\sum_{i=1}^m f_{\mathbf{w}_i}(\mathbf{x}_j)  -   \mathbb{E}_{\mathbf{W}\sim \alpha}\bigl[f_\mathbf{W}(\mathbf{x}_j)\bigr]  \biggr| < \epsilon.
 \end{equation}
This is because due to the definition of a $\delta$-covering for any $||\mathbf{x}||_2\le R$ there exists $\mathbf{x}_j$ for which $|| \mathbf{x}_j - \mathbf{x}|| \le \frac{\epsilon}{4L}$. Then, since any $f_\mathbf{w}$ is supposed to be $L$-Lipschitz we have
\begin{equation}
\biggl| \, \frac{1}{m}\sum_{i=1}^m f_{\mathbf{w}_i}(\mathbf{x}_j) -  \frac{1}{m}\sum_{i=1}^m f_{\mathbf{w}_i}(\mathbf{x})\, \biggr| \le \frac{\epsilon}{4},\quad \biggl| \mathbb{E}_{\mathbf{W}\sim \alpha}\bigl[f_\mathbf{W}(\mathbf{x}_j)\bigr] - \mathbb{E}_{\mathbf{W}\sim \alpha}\bigl[f_\mathbf{W}(\mathbf{x})\bigr]  \biggr| \le \frac{\epsilon}{4}
\end{equation}  
which together with \eqref{Thm 3 proof_ eq 1} shows \eqref{Thm 3 proof_ to show}. Hence, if  we choose
\begin{equation}
m = \frac{8M^2}{\epsilon^2}\bigl(\, k\log(12LR/\epsilon) +\log 2 \, \bigr) = \mathcal{O}\bigl( \frac{M^2 k \log (LR/\epsilon )}{\epsilon^2} \bigr)
\end{equation}
there will be some weight assignments $(\mathbf{w}_i)_{i=1}^m$ such that their uniform combination $\frac{1}{m}\sum_{i=1}^m f_{\mathbf{w}_i}(\mathbf{x})$ $\epsilon$-approximates the convex combination $f_{\alpha}$ uniformly over $\{\mathbf{x}:\, ||\mathbf{x}||_2\le R \}$.

\subsection{Proof of Theorem 5}
We show that for any distributions $P_0,P_1,P_2$ the following holds
\begin{equation}\label{Thm 4, eq 1 to show}
\bigl| d_{f,W_1}(P_0,P_2) - d_{f,W_1}(P_1,P_2) \bigr| \le W_1(P_0,P_1).
\end{equation}
The above inequality holds since if $Q_0$ and $Q_1$ solve the minimum sum optimization problems for $d_{f,W_1}(P_0,P_2)$, $d_{f,W_1}(P_1,P_2)$, we have
\begin{align*}
&d_{f,W_1}(P_0,P_2) - d_{f,W_1}(P_1,P_2) \le W_1(P_0,Q_1) - W_1(P_1,Q_1) \le W_1(P_0,P_1), \\
& d_{f,W_1}(P_1,P_2) - d_{f,W_1}(P_0,P_2) \le W_1(P_1,Q_0) - W_1(P_0,Q_0) \le W_1(P_0,P_1)
\end{align*}
where the second inequalities in both these lines follow from the symmetricity and triangle inequality property of the $W_1$-distance. Therefore, the following holds for any $Q$: 
$$\bigl| d_{f,W_1}(P_{G_{\boldsymbol{\theta}}(\mathbf{Z})},Q) - d_{f,W_1}(P_{G_{\boldsymbol{\theta}'}(\mathbf{Z})},Q) \bigr| \le W_1(P_{G_{\boldsymbol{\theta}}(\mathbf{Z})},P_{G_{\boldsymbol{\theta}'}(\mathbf{Z})}).$$ 
Hence, we only need to show $W_1(P_{G_{\boldsymbol{\theta}}(\mathbf{Z})},Q)$ is changing continuously with $\boldsymbol{\theta}$ and is almost everywhere differentiable. We prove these things using a similar proof to \cite{arjovsky2017wasserstein}'s proof for the continuity of the first-order Wasserstein distance. 

Consider two functions $G_{\boldsymbol{\theta}}, \, G_{\boldsymbol{\theta}'}$. The joint distribution $M$ for $(G_{\boldsymbol{\theta}}(\mathbf{Z}),G_{\boldsymbol{\theta}'}(\mathbf{Z}))$ is contained in $\Pi(P_{G_{\boldsymbol{\theta}}(\mathbf{Z})},P_{G_{\boldsymbol{\theta}'}(\mathbf{Z}))})$, which results in
\begin{align}
W_1\bigl(\, P_{G_{\boldsymbol{\theta}}(\mathbf{Z})}\, ,\, P_{G_{\boldsymbol{\theta}'}(\mathbf{Z})}\, \bigr) &\le \mathbb{E}_{M}[ \Vert \mathbf{X} - \mathbf{X}' \Vert]  \nonumber \\
& = \mathbb{E} \bigl[ \bigl\Vert \, G_{\boldsymbol{\theta}}(\mathbf{Z})\, - \, G_{\boldsymbol{\theta}'}(\mathbf{Z})\,\bigr\Vert \bigr]. \label{Theorem 4 proof: Eq 1}
\end{align}
If we let $\boldsymbol{\theta}'\rightarrow \boldsymbol{\theta}$ then $G_{\boldsymbol{\theta}}(\mathbf{z})\rightarrow G_{\boldsymbol{\theta}'}(\mathbf{z})$ and hence $\Vert \, G_{\boldsymbol{\theta}'}(\mathbf{z}) - G_{\boldsymbol{\theta}}(\mathbf{z})\,\Vert \rightarrow 0$ hold pointwise. Since $\mathcal{X}$ is assumed to be compact, there exists some finite $R$ for which $0\le \Vert \mathbf{x} - \mathbf{x}'\Vert\le R$ holds over the compact $\mathcal{X}\times \mathcal{X}$. Then the bounded convergence theorem implies  $\mathbb{E} \bigl[\, \bigl\Vert G_{\boldsymbol{\theta}}(\mathbf{Z}) - G_{\boldsymbol{\theta}'}(\mathbf{Z}) \bigr\Vert\, \bigr]$ converges to $0$ as $\boldsymbol{\theta}'\rightarrow \boldsymbol{\theta}$. Then, since $W_1$-distance always takes non-negative values
$$W_1\bigl(\, P_{G_{\boldsymbol{\theta}}(\mathbf{Z})}\, ,\, P_{G_{\boldsymbol{\theta}'}(\mathbf{Z})}\, \bigr)\,  \xrightarrow{\boldsymbol{\theta}'\rightarrow \boldsymbol{\theta}}\, 0.$$ 
Thus, $W_1$ satisfies the discussed continuity property and as a result $d_{f,W_1}(P_{G_{\boldsymbol{\theta}}(\mathbf{Z})},Q)$ changes continuously with $\boldsymbol{\theta}$. Furthermore, if $G_{\boldsymbol{\theta}}$ is locally-Lipschitz and its Lipschitz constant w.r.t. parameters $\boldsymbol{\theta}$ is bounded above by $L$,
\begin{align}
d_{f,W_1}\bigl(\, P_{G_{\boldsymbol{\theta}}(\mathbf{Z})}\, ,\, P_{G_{\boldsymbol{\theta}'}(\mathbf{Z})}\, \bigr) &\le W_1\bigl(\, P_{G_{\boldsymbol{\theta}}(\mathbf{Z})}\, ,\, P_{G_{\boldsymbol{\theta}'}(\mathbf{Z})}\, \bigr) \nonumber \\
 &\le  \mathbb{E} \bigl[ \bigl\Vert \, G_{\boldsymbol{\theta}}(\mathbf{Z})\, - \, G_{\boldsymbol{\theta}'}(\mathbf{Z})\,\bigr\Vert \bigr] \nonumber \\
& \le L \Vert \boldsymbol{\theta} - \boldsymbol{\theta}'  \Vert ,
\end{align}
which implies both $W_1(P_{G_{\boldsymbol{\theta}}(\mathbf{Z})},Q)$ and  $d_{f,W_1}(P_{G_{\boldsymbol{\theta}}(\mathbf{Z})},Q)$ are everywhere continuous and almost everywhere differentiable w.r.t. $\boldsymbol{\theta}$.

\subsection{Proof of Theorem 6}
We first generalize the definition of the hybrid divergence to a general minimum-sum hybrid of an f-divergence and an optimal transport cost. For f-divergence $d_f$ and optimal transport cost $OT_c$ corresponding to convex function $f$ and cost $c$ respectively, we define the following hybrid $d_{f,c}$ of the two divergence measures:
\begin{equation}
d_{f,c}(P_1 , P_2) := \inf_{Q}\, OT_c(P_1,Q) + d_f(Q,P_2).
\end{equation}
\begin{lemma} \label{Lemma Kantorovich hybrid}
Given a symmetric f-divergence $d_f$ with convex lower semicontinuous $f$ and a non-negative lower semicontinuous  $c$, $d_{f,c}(P_1 , P_2)$ will be a convex function of $P_1$ and $P_2$, and further satisfies the following generalization of the Kantorovich duality \cite{villani2008optimal}:
\begin{align}
d_{f,c}(P_1,P_2) = \sup_{D\, \text{c-concave} }\: \mathbb{E}_{P_1}[D(\mathbf{X})]  - \mathbb{E}_{P_2}[f^*(D^c(\mathbf{X}))].
\end{align}
\end{lemma}
\begin{proof}
According to the Kantorovich duality \cite{villani2008optimal} we have
\begin{align*}
d_{f,c}(P_1,P_2) &\stackrel{(a)}{=} \inf_Q\: OT_c(P_1,Q) + d_f(Q,P_2) \\
 &\stackrel{(b)}{=} \inf_Q\:\sup_{D\, \text{c-concave}}\: \mathbb{E}_{P_1}[D(\mathbf{X})]- \mathbb{E}_{Q}[D^c(\mathbf{X})] + d_f(Q,P_2) \\
  &\stackrel{(c)}{=} \inf_Q\:\sup_{D\, \text{c-concave}}\: \mathbb{E}_{P_1}[D(\mathbf{X})]- \mathbb{E}_{Q}[D^c(\mathbf{X})] + d_f(P_2,Q) \\
 &\stackrel{(d)}{=}\sup_{D\, \text{c-concave}}\: \inf_Q\: \mathbb{E}_{P_1}[D(\mathbf{X})]- \mathbb{E}_{Q}[D^c(\mathbf{X})] + d_f(P_2,Q) \\
  &\stackrel{}{=}\sup_{D\, \text{c-concave}}\: \mathbb{E}_{P_1}[D(\mathbf{X})] +\, \inf_Q\: d_f(P_2,Q) -\mathbb{E}_{Q}[D^c(\mathbf{X})] \\
   &\stackrel{(e)}{=}\sup_{D\, \text{c-concave}}\: \mathbb{E}_{P_1}[D(\mathbf{X})] -\, {d_f}^*_{P_2} (D^c)  \\
    &\stackrel{(f)}{=}\sup_{D\, \text{c-concave}}\: \mathbb{E}_{P_1}[D(\mathbf{X})] +\, \max_{\lambda\in\mathbb{R}}\: \lambda - \mathbb{E}_{P_2}[f^*(D^c(\mathbf{X})+\lambda)]  \\
    &\stackrel{}{=}\sup_{D\, \text{c-concave},\,\lambda\in\mathbb{R}}\: \mathbb{E}_{P_1}[D(\mathbf{X})+ \lambda]  - \mathbb{E}_{P_2}[f^*(D^c(\mathbf{X})+\lambda)].  \\
    &\stackrel{}{=}\sup_{D\, \text{c-concave}}\: \mathbb{E}_{P_1}[D(\mathbf{X})]  - \mathbb{E}_{P_2}[f^*(D^c(\mathbf{X}))].
\end{align*}
Here (a) holds according to the definition. (b) is a  consequence of the Kantorovich duality (\cite{villani2008optimal}, Theorem 5.10). (c) holds becuase $d_f$ is assumed to be symmetric. (d) holds due to the generalized minimax theorem \cite{borwein2016very}, since the space of distributions over compact $\mathcal{X}$ is convex and weakly compact, the set of c-concave functions is convex, the minimax objective is concave in $D$ and convex in $Q$. (e) holds according to the conjugate $d^*_P$'s definition, and (f) is based on our earlier result in \eqref{df-conjuagte-max}. Note that the final expression is maximizing an objective linear in $P_2$, which is convex in $P_2$. The last equality holds since for any constant $\lambda\in\mathbb{R}$ if $D^c$ is the c-transform of $D$, $D^c + \lambda$ will be the c-transform of $D+\lambda$. Finally, note that $d_{f,c}(P_1,P_2)$ is the supremum of some linear functions of $P_1$ and $P_2$ with compact support sets. Hence $d_{f,c}$ will be a convex function of $P_1$ and $P_2$.
\end{proof}
Now we prove the following generalization of Theorem 6, which directly results in Theorem 6 for the difference norm cost $c_1(\mathbf{x},\mathbf{x}') = \Vert \mathbf{x} - \mathbf{x}' \Vert$. Here note that for cost $c_1$ the c-transform of a $1$-Lipschitz function $D$ will be $D$ itself, which implies if $f^*\circ D$ is $1$-Lipschitz then 
\begin{align*}
 -f^* ( D(G(\mathbf{Z})))  = \inf_{\mathbf{x}'}\, -f^* ( D(\mathbf{x}')) + c_1\bigl(G(\mathbf{Z})\, ,\,\mathbf{x}'\bigr).
\end{align*}
\begin{thm*}[Generalization of Theorem 6]
Assume $d_f$ is a symmetric f-divergence, i.e. $d_f(P,Q)=d_f(Q,P)$, satisfying the assumptions in Lemma 2. Suppose $\mathcal{F}$ is a convex set of continuous functions closed to constant additions and cost function $c$ is non-negative and continuous. Then, the minimax problem in Theorem 1 and Corollary 1 for the mixed divergence $d_{f,c}$ reduces to
\begin{align}\label{fW-GAN dual Appendix}
\min_{ G\in \mathcal{G} }\;  \max_{D \in \mathcal{F}}\;  \mathbb{E}_{P_\mathbf{X}}[D({\mathbf{X}})] + \mathbb{E}\bigl[\, \inf_{\mathbf{x}'}\, -f^* ( D(\mathbf{x}')) + c\bigl(G(\mathbf{Z})\, ,\,\mathbf{x}'\bigr)  \,\bigr].
\end{align}
\end{thm*}
\begin{proof}
Accoriding to Lemma \ref{Lemma Kantorovich hybrid}, $d_{f,c}(P,Q)$ satisfies the convexity property in $Q$. Hence, the assumptions of Theorem 1 and Corollary 1 hold and we only need to plug in the conjugate ${d_{f,c}}^*_{P_1}$ into Corollary 1. According to the definition,
\begin{align*}
{d_{f,c}}^*_{P_1} (D) &= \sup_{P_2}\: \mathbb{E}_{P_2}[D(\mathbf{X})] - d_{f,c}(P_1,P_2)  \\
& \stackrel{}{=} \sup_{P_2}\:\sup_Q \: -OT_c(P_1,Q) - d_{f}(Q,P_2) + \mathbb{E}_{P_2}[D(\mathbf{X})] \\
& = \sup_Q \: \sup_{P_2}\: -OT_c(P_1,Q) - d_{f}(Q,P_2) + \mathbb{E}_{P_2}[D(\mathbf{X})] \\
& = \sup_Q \: - OT_c(P_1,Q) + \sup_{P_2}\:  \mathbb{E}_{P_2}[D(\mathbf{X})] - d_{f}(Q,P_2)\\
& = \sup_Q \: -OT_c(P_1,Q) + {d_f}^*_Q(D) \\
& \stackrel{(g)}{=} \sup_Q \: - OT_c(P_1,Q) + \min_{\lambda\in\mathbb{R}}\: -\lambda + \mathbb{E}_Q[f^*(D(\mathbf{X}) + \lambda)] \\
& = \sup_Q \: \min_{\lambda\in\mathbb{R}}\: -OT_c(P_1,Q) -  \lambda + \mathbb{E}_Q[f^*(D(\mathbf{X}) + \lambda)] \\
& \stackrel{(h)}{=} \min_{\lambda\in\mathbb{R}}\:\sup_Q \:   -OT_c(P_1,Q) -  \lambda + \mathbb{E}_Q[f^*(D(\mathbf{X}) + \lambda)] \\
& \stackrel{(i)}{=} \inf_{\lambda\in\mathbb{R}}\:  -\lambda + \mathbb{E}_{P_1}\bigl[\bigl(f^*\circ (D+ \lambda)\bigr)^c(\mathbf{X} )\bigr].
\end{align*}
Here (g) holds based on our earlier result in \eqref{df-conjuagte-max}. (h) is a consequence of the minimax theorem, since the space of distributions over compact $\mathcal{X}$ is convex and compact, and the objective is concave in $\lambda$ and lower semicontinuous and convex in $Q$. (i) is implied by Lemma 3. Therefore, according to Corollary 1
\begin{align*}
  & \min_{P_{G(\mathbf{Z})}\in\mathcal{P_G}}\; \min_{Q_\mathbf{X}}\; d_{f,c}( P_{G(\mathbf{Z})} , Q) + \max_{D\in\mathcal{F}}\,\bigl\{ \mathbb{E}_{P_\mathbf{X}}[D(\mathbf{X})] - \mathbb{E}_{Q}[D(\mathbf{X})]  \bigr\} \nonumber \\
 = &\;\;\;\min_{ G\in \mathcal{G} }\;  \max_{D \in \mathcal{F}}\;  \mathbb{E}_{P_\mathbf{X}}[D({\mathbf{X}})] -{d_{f,c}}^*_{P_{G(\mathbf{Z})}}\,(D)  \\
  = &\;\;\;\min_{ G\in \mathcal{G} }\;  \max_{D \in \mathcal{F}}\;  \mathbb{E}_{P_\mathbf{X}}[D({\mathbf{X}})] +\max_{\lambda\in\mathbb{R}}\:  \lambda - \mathbb{E}\bigl[\bigl(f^*\circ (D+ \lambda)\bigr)^c(G(\mathbf{Z}) )\bigr] \\
    = &\;\;\;\min_{ G\in \mathcal{G} }\;  \max_{D \in \mathcal{F},\, \lambda\in\mathbb{R}}\;  \mathbb{E}_{P_\mathbf{X}}[D({\mathbf{X}})+ \lambda]  - \mathbb{E}\bigl[\bigl(f^*\circ (D+ \lambda)\bigr)^c(G(\mathbf{Z}) )\bigr] \\
       \stackrel{(j)}{=} &\;\;\;\min_{ G\in \mathcal{G} }\;  \max_{D \in \mathcal{F}}\;  \mathbb{E}_{P_\mathbf{X}}[D({\mathbf{X}})]  - \mathbb{E}\bigl[(f^*\circ D)^c(G(\mathbf{Z}) )\bigr] \\
 = &\;\;\;\min_{ G\in \mathcal{G} }\;  \max_{D \in \mathcal{F}}\;  \mathbb{E}_{P_\mathbf{X}}[D({\mathbf{X}})] - \mathbb{E}\bigl[\, \sup_{\mathbf{x}'}\, f^* ( D(\mathbf{x}')) - c\bigl(G(\mathbf{Z})\, ,\,\mathbf{x}'\bigr)\,\bigr] \\
 = &\;\;\;\min_{ G\in \mathcal{G} }\;  \max_{D \in \mathcal{F}}\;  \mathbb{E}_{P_\mathbf{X}}[D({\mathbf{X}})] + \mathbb{E}\bigl[\, \inf_{\mathbf{x}'}\, -f^* ( D(\mathbf{x}')) + c\bigl(G(\mathbf{Z})\, ,\,\mathbf{x}'\bigr)\,\bigr]. 
\end{align*}
Here (j) holds since $\mathcal{F}$ is assumed to be closed to constant additions. Hence, the proof is complete.
\end{proof}

\subsection{Proof of Theorem 7}
Consider distributions $P_0,P_1,P_2$. Let $Q_0,Q_1$ be the optimal solutions to the minimum sum optimization problems for $d_{f,W_2}(P_0,P_2)$ and $d_{f,W_2}(P_1,P_2)$, respectively. Then, according to the definition
\begin{align*}
&d_{f,W_2}(P_0,P_2) - d_{f,W_2}(P_1,P_2) \le W^2_2(P_0,Q_1) - W^2_2(P_1,Q_1) , \\ &d_{f,W_2}(P_1,P_2) - d_{f,W_2}(P_0,P_2) \le W^2_2(P_1,Q_0) - W^2_2(P_0,Q_0) 
\end{align*}
which implies
\begin{align*}
\bigl|\, d_{f,W_2}(P_0,P_2) - d_{f,W_2}(P_1,P_2) \bigr| \le \sup_Q\: \bigl| W^2_2(P_0,Q) - W^2_2(P_1,Q)\bigr|.
\end{align*}
Hence, for $G_{\boldsymbol{\theta}},\, G_{\boldsymbol{\theta}'}$ and any distribution $P_2$ we have
\begin{align}
\bigl|\, d_{f,W_2}(P_{G_{\boldsymbol{\theta}}(\mathbf{Z})},P_2) - d_{f,W_2}(P_{G_{\boldsymbol{\theta}'}(\mathbf{Z})},P_2) \, \bigr| \le \sup_Q\: \bigl| W^2_2(P_{G_{\boldsymbol{\theta}}(\mathbf{Z})},Q) - W^2_2(P_{G_{\boldsymbol{\theta}'}(\mathbf{Z})},Q)\bigr|. \label{Thm 5, eq 2 to use}
\end{align}
Fix a distribution $Q$ over the compact $\mathcal{X}$. Then, for any $(G_{\boldsymbol{\theta}}(\mathbf{Z}),\mathbf{X}')$ whose joint distribution is in $\Pi(P_{G_{\boldsymbol{\theta}}(\mathbf{Z})},Q)$, $(G_{\boldsymbol{\theta}'}(\mathbf{Z}),\mathbf{X}')$ has a joint distribution in $\Pi(P_{G_{\boldsymbol{\theta}'}(\mathbf{Z})},Q)$. Moreover, since $\mathcal{X}$ is a compact set in a Hilbert space, any $\mathbf{x}\in\mathcal{X}$ is norm-bounded for some finite $R$ as ${\Vert}\mathbf{x} {\Vert}\le R$, which implies
\begin{align}
&\biggl|\, W^2_2(P_{G_{\boldsymbol{\theta}}(\mathbf{Z})},Q) - W^2_2(P_{G_{\boldsymbol{\theta}'}(\mathbf{Z})},Q)\, \biggr| \nonumber\\
\le\; & \sup_{M_{\mathbf{Z},\mathbf{X}'}\in \Pi(P_{\mathbf{Z}},Q)}\:\biggl|\, \mathbb{E}_M\biggl[ \, {\Vert}\, G_{\boldsymbol{\theta}}(\mathbf{Z})-\mathbf{X}' \,{\Vert}^2 - {\Vert}\, G_{\boldsymbol{\theta}'}(\mathbf{Z})-\mathbf{X}' \,{\Vert}^2\, \biggr]\,\biggr| \nonumber \\
\le \; &\sup_{M_{\mathbf{Z},\mathbf{X}'}\in \Pi(P_{\mathbf{Z}},Q)}\: \mathbb{E}_M\biggl[ \, \left|\, {\Vert} G_{\boldsymbol{\theta}}(\mathbf{Z}) {\Vert}^2 - {\Vert}G_{\boldsymbol{\theta}'}(\mathbf{Z}){\Vert}^2 \, \right| + 2{\Vert}\mathbf{X}' {\Vert} \, {\Vert} \, G_{\boldsymbol{\theta}'}(\mathbf{Z})-G_{\boldsymbol{\theta}}(\mathbf{Z}) \,{\Vert}\, \biggr] \nonumber \\
 \le \; & \mathbb{E}_{P_{\mathbf{Z}}}\biggl[\, \left|\, {\Vert} G_{\boldsymbol{\theta}}(\mathbf{Z}) {\Vert}^2 - {\Vert}G_{\boldsymbol{\theta}'}(\mathbf{Z}){\Vert}^2 \, \right| + 2 R \, {\Vert} \, G_{\boldsymbol{\theta}'}(\mathbf{Z})-G_{\boldsymbol{\theta}}(\mathbf{Z}) \,{\Vert} \, \biggr].\nonumber
\end{align}
Taking a supremum over $Q$ from both sides of the above inequality shows
\begin{align}
&\sup_Q\: \biggl|\, W^2_2(P_{G_{\boldsymbol{\theta}}(\mathbf{Z})},Q) - W^2_2(P_{G_{\boldsymbol{\theta}'}(\mathbf{Z})},Q)\, \biggr| \nonumber\\
 \le \; & \mathbb{E}_{P_{\mathbf{Z}}}\biggl[\, \left|\, {\Vert} G_{\boldsymbol{\theta}}(\mathbf{Z}) {\Vert}^2 - {\Vert}G_{\boldsymbol{\theta}'}(\mathbf{Z}){\Vert}^2 \, \right| + 2 R \, {\Vert} \, G_{\boldsymbol{\theta}'}(\mathbf{Z})-G_{\boldsymbol{\theta}}(\mathbf{Z}) \,{\Vert} \, \biggr]. \label{Thm 5 to show eq 1}
\end{align}
Since $G_{\boldsymbol{\theta}}$ changes continuously with $\boldsymbol{\theta}$, $\bigl|\, {\Vert} G_{\boldsymbol{\theta}}(\mathbf{z}) {\Vert}^2 - {\Vert}G_{\boldsymbol{\theta}'}(\mathbf{z}){\Vert}^2 \, \bigr| + 2 R \, {\Vert} \, G_{\boldsymbol{\theta}'}(\mathbf{z})-G_{\boldsymbol{\theta}}(\mathbf{z}) \,{\Vert} \rightarrow 0$ as $\boldsymbol{\theta}'\rightarrow \boldsymbol{\theta}$ holds pointwise. Therefore, since $\mathcal{X}$ is compact and hence bounded, the bounded convergence theorem together with \eqref{Thm 5 to show eq 1} implies
\begin{align}
\sup_Q\: \biggl|\, W^2_2(P_{G_{\boldsymbol{\theta}}(\mathbf{Z})},Q) - W^2_2(P_{G_{\boldsymbol{\theta}'}(\mathbf{Z})},Q)\, \biggr|
\, \xrightarrow{\boldsymbol{\theta}'\rightarrow \boldsymbol{\theta}}\, 0. \label{Thm5 eq 2 to use}
\end{align}
Now, combining \eqref{Thm 5, eq 2 to use} and \eqref{Thm5 eq 2 to use} shows for any distribution $P_2$
\begin{align}
 \biggl|\, d_{f,W_2}(P_{G_{\boldsymbol{\theta}}(\mathbf{Z})},P_2) -  d_{f,W_2}(P_{G_{\boldsymbol{\theta}'}(\mathbf{Z})},P_2)\, \biggr|
\, \xrightarrow{\boldsymbol{\theta}'\rightarrow \boldsymbol{\theta}}\, 0.
\end{align}
Also, if we further assume $G_{\boldsymbol{\theta}}$ is bounded by $T$ locally-Lipschitz w.r.t. $\boldsymbol{\theta}$ with Lipschitz constant $L$, then 
\begin{align}
&\sup_Q\: \biggl|\, W^2_2(P_{G_{\boldsymbol{\theta}}(\mathbf{Z})},Q) - W^2_2(P_{G_{\boldsymbol{\theta}'}(\mathbf{Z})},Q)\, \biggr| \nonumber\\
 \le \; & \mathbb{E}_{P_{\mathbf{Z}}}\biggl[\, \left|\, {\Vert} G_{\boldsymbol{\theta}}(\mathbf{Z}) {\Vert}^2 - {\Vert}G_{\boldsymbol{\theta}'}(\mathbf{Z}){\Vert}^2 \, \right| + 2 R \, {\Vert} \, G_{\boldsymbol{\theta}'}(\mathbf{Z})-G_{\boldsymbol{\theta}}(\mathbf{Z}) \,{\Vert} \, \biggr] \label{Thm 5 to show eq 2} \\
 \le \; & \mathbb{E}_{P_{\mathbf{Z}}}\biggl[\, \left|\, ({\Vert} G_{\boldsymbol{\theta}}(\mathbf{Z}) {\Vert} + {\Vert}G_{\boldsymbol{\theta}'}(\mathbf{Z}){\Vert} ) \, ({\Vert} G_{\boldsymbol{\theta}}(\mathbf{Z}) {\Vert} - {\Vert}G_{\boldsymbol{\theta}'}(\mathbf{Z}){\Vert} )\, \right| + 2 R \, {\Vert} \, G_{\boldsymbol{\theta}'}(\mathbf{Z})-G_{\boldsymbol{\theta}}(\mathbf{Z}) \,{\Vert} \, \biggr] \nonumber \\
 \le \; & \mathbb{E}_{P_{\mathbf{Z}}}\biggl[\,  2T \left|\,  {\Vert} G_{\boldsymbol{\theta}}(\mathbf{Z}) {\Vert} - {\Vert}G_{\boldsymbol{\theta}'}(\mathbf{Z}){\Vert} \, \right| + 2 R \, {\Vert} \, G_{\boldsymbol{\theta}'}(\mathbf{Z})-G_{\boldsymbol{\theta}}(\mathbf{Z}) \,{\Vert} \, \biggr] \nonumber \\
 \le \; & \mathbb{E}_{P_{\mathbf{Z}}}\biggl[\,  2(T+R) \, {\Vert} \, G_{\boldsymbol{\theta}'}(\mathbf{Z})-G_{\boldsymbol{\theta}}(\mathbf{Z}) \,{\Vert} \, \biggr] \nonumber \\
 \le \; &  2(T+R) L \, {\Vert} \, \boldsymbol{\theta}' - \boldsymbol{\theta} \,{\Vert} , \nonumber
\end{align}
implying $d_{f,W_2}(P_{G_{\boldsymbol{\theta}}(\mathbf{Z})},Q)$ is continuous everywhere and differentiable almost everywhere  as a function of $\boldsymbol{\theta}$.

\section{Proof of Theorem 8}
Note that applying the generalized version of Theorem 6 proved in the Appendix to difference norm-squared cost $c_2(\mathbf{x},\mathbf{x}')=\Vert \mathbf{x}- \mathbf{x}'\Vert^2$ reveals that for a symmetric f-divergence $d_f$ and convex set $\mathcal{F}$ closed to constant additions the minimax problem in Theorem 1 and Corollary 1 for the mixed divergence $d_{f,c_2}$ reduces to
\begin{align}\label{fW-GAN dual Appendix-2}
&\min_{ G\in \mathcal{G} }\;  \max_{D \in \mathcal{F}}\;  \mathbb{E}_{P_\mathbf{X}}[D({\mathbf{X}})] + \mathbb{E}\bigl[\, \min_{\mathbf{x}'}\, -f^* ( D(\mathbf{x}')) + c_2\bigl(G(\mathbf{Z})\, ,\,\mathbf{x}'\bigr)  \,\bigr] \nonumber \\
= \, & \min_{ G\in \mathcal{G} }\;  \max_{D \in \mathcal{F}}\;  \mathbb{E}_{P_\mathbf{X}}[D({\mathbf{X}})] + \mathbb{E}\bigl[\, \min_{\mathbf{x}'}\, -f^* ( D(\mathbf{x}')) + \bigl\Vert G(\mathbf{Z})\, -\,\mathbf{x}'\bigr\Vert^2  \,\bigr]
\\
= \, & \min_{ G\in \mathcal{G} }\;  \max_{D \in \mathcal{F}}\;  \mathbb{E}_{P_\mathbf{X}}[D({\mathbf{X}})] + \mathbb{E}\bigl[\, \min_{\mathbf{u}}\, -f^* \bigl( D(\,  G(\mathbf{Z})+\mathbf{u}\, )\bigr) + \bigl\Vert \mathbf{u}\bigr\Vert^2  \,\bigr]. \nonumber
\end{align}
Here the last equality follows the change of variable $\mathbf{u} = \mathbf{x}' - G(\mathbf{Z})$. Also, note that $d_{f,W_2}$ defined in the main text is the same as the special case of the generalized hybrid divergence $d_{f,c}$ with cost $c_2$. Hence, the proof is complete. 
\subsection{Two additional examples for convex duality framework applied to Wasserstein distances}
\subsubsection{Total variation distance: Energy-based GAN}
Consider the total variation distance $\delta(P,Q)$ which is defined as
\begin{equation}
\delta(P,Q) := \sup_{A\in\Sigma}\, \bigl| P(A) - Q(A) \bigr|,
\end{equation}
where $\Sigma$ is the set all Borel subsets of support set $\mathcal{X}$. More generally we consider $\delta_m(P,Q) = m \delta(P,Q)$ for any positive $m>0$. Under mild assumptions, the total variation distance can be cast as a Wasserstein distance for the indicator cost $c_{m,I}(\mathbf{x},\mathbf{x}')=m\,\mathbb{I}(\mathbf{x}\neq\mathbf{x}')$ \cite{villani2008optimal}, i.e. $\delta_m(P,Q) = OT_{c_{m,I}}(P,Q)$. Note that $c_{m,I}$ is a lower semicontinuous distance function, and hence Lemma 3 applies to $c_{m,I}$ indicating
\begin{align*}
{\delta_m}^*_P(D) &= {OT_{c_{I,m}}}^*_P(D)\\
 &= \mathbb{E}_P[D^{c_{I,m}}(\mathbf{X})] \\
& =\mathbb{E}_P\bigl[\sup_{\mathbf{x}'}\: D(\mathbf{x}') - m\,{c_I}(\mathbf{X},\mathbf{x}')\bigr] \\
& =\mathbb{E}_P\bigl[\, \max\, \bigl\{\,D(\mathbf{X})\, ,\,\max_{\mathbf{x}'}\, D(\mathbf{x}') - m \bigr\}\,\bigr] \\
& =\mathbb{E}_P\bigl[\, \max\, \bigl\{\, m+D(\mathbf{X})-\max_{\mathbf{x}'}\, D(\mathbf{x}')\, ,\, 0 \bigr\}\, \bigr] + \max_{\mathbf{x}'}\, D(\mathbf{x}') -m 
\end{align*}
Without loss of generality, we can assume that the maximum discriminator output is always $0$ which results in
\begin{align*}
{\delta_m}^*_P(D)  = \mathbb{E}_P\bigl[\, \max\, \bigl\{\, m+D(\mathbf{X})\, ,\, 0 \bigr\}\, \bigr] -m 
\end{align*}
Therefore, the minimax problem in Corollaries 1,2 for the total variation distance will be
\begin{align*}
&\min_{G\in\mathcal{G}}\: \max_{D\in\mathcal{F}}\: \mathbb{E}_P[D(\mathbf{X})] - {\delta_m}^*_P(D)  \\
=\, & \min_{G\in\mathcal{G}}\: \max_{D\in\mathcal{F}}\: \mathbb{E}_P[D(\mathbf{X})] - \mathbb{E}_P\bigl[\, \max \bigl\{m+D(G(\mathbf{Z}))\, ,\, 0 \bigr\}\, \bigr] + m\\
=\, & \min_{G\in\mathcal{G}}\: \max_{-D\in\mathcal{F}}\: -\mathbb{E}_P[D(\mathbf{X})] - \mathbb{E}_P\bigl[\, \max \bigl\{m-D(G(\mathbf{Z}))\, ,\, 0 \bigr\}\, \bigr] + m\\
=\, & \min_{G\in\mathcal{G}}\: \max_{\tilde{D}\in \mathcal{F}}\: - \mathbb{E}_P[ \tilde{D}(\mathbf{X})] - \mathbb{E}_P\bigl[\, \max\bigl\{ m-\tilde{D}(G(\mathbf{Z}))\, ,\, 0 \bigr\}\, \bigr] + m
\end{align*}
where the last equality follows from the assumption that for any $D\in\mathcal{F}$ we have $-D\in\mathcal{F}$. Since $D$ is assumed to be non-positive, $\tilde{D}$ takes non-negative values. Note that this problem is equivalent to a minimax game where discriminator $D$ is \emph{minimizing} the following cost over $\mathcal{F}$:
\begin{equation}
L_D(G,D) = \mathbb{E}_P[ D(\mathbf{X})] + \mathbb{E}_P\bigl[\, \max\, \bigl\{\, m-D(G(\mathbf{Z}))\, ,\, 0 \bigr\}\, \bigr] 
\end{equation}
which is also the discriminator cost function in the energy-based GAN \cite{zhao2016energy}. Hence, for any fixed $G\in \mathcal{G}$, the optimal discriminator $D\in\mathcal{F}$ for the total variation's minimax problem is the same as the energy-based GAN's optimal discriminator.
\subsubsection{Second-order Wasserstein distance: the LQG setting}
Consider the second-order Wasserstein distance $W_2(P,Q)$, and suppose $\mathcal{F}$ is the set of quadratic functions over $\mathbf{X}$, which is a linear space. Also assume the generator $G$ is a linear function and the $r$-dimensional noise $\mathbf{Z}$ is Gaussianly-distributed with zero-mean and identity covariance matrix $I_{r\times r}$. According to the interpretation provided in Corollary 2, the second-order Wasserstein GAN finds the multivariate Gaussian distribution with rank $r$ covariance matrix minimizing the $W_2$ distance to the set of distributions with their second-order moments matched to $P_\mathbf{X}$'s moments. 

Since the value of $\mathbb{E}[{\Vert}\mathbf{X}-G(\mathbf{Z}){\Vert}^2]$ depends only on the second-order moments of the vector $[\mathbf{X},G(\mathbf{Z})]$, we can minimize the $W_2$-distance between the two sets by minimizing this expectation over Gaussianly-distributed vectors $[\mathbf{X},G(\mathbf{Z})]$ subject to a rank $r$ covariance matrix for $[G(\mathbf{Z})]$ and a pre-determined covariance matrix for $[\mathbf{X}]$. Hence, the optimal $G^*$ simply corresponds to the $r$-PCA solution for $P_\mathbf{X}$. 

This example shows Theorem 3 provides another way to recover \cite{feizi2017understanding}'s main result under the linear generator, quadratic discriminator and Gaussianly-distributed data assumptions.

\end{document}